\documentclass{article}

\usepackage{enumitem} 
\usepackage{times}
\usepackage{graphicx} 
\usepackage{subfigure} 

\usepackage{natbib}

\usepackage{algorithm}
\usepackage{algorithmic}

\usepackage{hyperref}
\usepackage{breakurl}

\AtBeginShipout{%
  \ifnum\value{page}>1 %
    \typeout{* Additional boxing of page `\thepage'}%
    \setbox\AtBeginShipoutBox=\hbox{\copy\AtBeginShipoutBox}%
  \fi
}

\usepackage[accepted]{template/icml2017}

\usepackage{pifont}

\usepackage{epsfig}
\usepackage{amssymb}
\usepackage{amsmath}
\usepackage{amsthm}
\usepackage{amsfonts}
\usepackage{bbding}
\usepackage{array}
\usepackage{caption,tabularx,booktabs}
\usepackage{arydshln}
\usepackage{xargs}                      
\usepackage[dvipsnames]{xcolor}  
\usepackage[colorinlistoftodos,prependcaption,textsize=tiny]{todonotes}
\newcommandx{\unsure}[2][1=]{\todo[inline,linecolor=red,backgroundcolor=red!25,bordercolor=red,#1]{#2}}
\newcommandx{\change}[2][1=]{\todo[linecolor=blue,backgroundcolor=blue!25,bordercolor=blue,#1]{#2}}
\newcommandx{\info}[2][1=]{\todo[linecolor=OliveGreen,inline,backgroundcolor=OliveGreen!25,bordercolor=OliveGreen,#1]{#2}}
\newcommandx{\improvement}[2][1=]{\todo[linecolor=Plum,inline,backgroundcolor=Plum!25,bordercolor=Plum,#1]{#2}}

\usepackage[font=small,skip=5pt]{caption}

\newtheorem{thm}{Theorem}
\newtheorem{lem}{Lemma}

\newtheorem{cor}{Corollary}

\newtheorem{ass}{Assumption}
\newtheorem{remark}{Remark}

\newcolumntype{C}[1]{>{\centering\let\newline\\\arraybackslash\hspace{0pt}}m{#1}}

\newcommand\tagthis{\addtocounter{equation}{1}\tag{\theequation}}
\DeclareMathOperator{\Exp}{\mathbb{E}}           
\DeclareMathOperator{\R}{\mathbb{R}} 
\DeclareMathOperator{\Ocal}{\mathcal{O}} 
\newcommand{\eqdef}{\stackrel{\text{def}}{=}}
\newcommand{\setn}{[n]}
\renewcommand{\top}{T}

\makeatletter
\newcounter{subthm} 
\let\savedc@thm\c@hyp
\newenvironment{subthm}
 {%
  \setcounter{subthm}{0}%
  \stepcounter{thm}%
  \edef\saved@thm{\thethm}
  \let\c@thm\c@subthm     
  \renewcommand{\thethm}{\saved@thm\alph{thm}}%
 }
 {}
\newcommand{\normhyp}{%
  \let\c@hyp\savedc@hyp 
  \renewcommand\thehyp{\arabic{hyp}}%
} 
\makeatother

\makeatletter
\newcounter{subass} 
\let\savedc@ass\c@hyp
\newenvironment{subass}
 {%
  \setcounter{subass}{0}%
  \stepcounter{ass}%
  \edef\saved@ass{\theass}
  \let\c@ass\c@subass     
  \renewcommand{\theass}{\saved@ass\alph{ass}}%
 }
 {}
\makeatother

\icmltitlerunning{SARAH: A Novel Method for Machine Learning Problems 
			Using Stochastic Recursive Gradient}

\begin{document} 

\twocolumn[
\icmltitle{SARAH: A Novel Method for Machine Learning Problems \\ 
			Using Stochastic Recursive Gradient}




\begin{icmlauthorlist}
\icmlauthor{Lam M. Nguyen}{to}
\icmlauthor{Jie Liu}{to}
\icmlauthor{Katya Scheinberg}{to,go}
\icmlauthor{Martin Tak\'{a}\v{c}}{to}
\end{icmlauthorlist}

\icmlaffiliation{to}{Department of Industrial and Systems Engineering,
            Lehigh University, USA.}
\icmlaffiliation{go}{On leave at The University of Oxford, UK. All authors were supported by NSF Grant CCF-1618717. Katya Scheinberg was partially supported by NSF Grants DMS 13-19356, CCF-1320137 and CCF-1618717}            

\icmlcorrespondingauthor{Lam M. Nguyen}{lamnguyen.mltd@gmail.com}
\icmlcorrespondingauthor{Jie Liu}{jie.liu.2018@gmail.com}
\icmlcorrespondingauthor{Katya Scheinberg}{katyas@lehigh.edu}
\icmlcorrespondingauthor{Martin Tak\'{a}\v{c}}{Takac.MT@gmail.com}

\icmlkeywords{boring formatting information, machine learning, ICML}

\vskip 0.3in
]



\printAffiliationsAndNotice{}  

\begin{abstract} 
In this paper, we propose a StochAstic Recursive grAdient algoritHm (SARAH), as well as its practical variant SARAH+, as a novel approach to the finite-sum minimization problems. Different from the vanilla SGD and other modern stochastic methods such as SVRG, S2GD, SAG and SAGA, SARAH admits a simple recursive framework for updating stochastic gradient estimates; when comparing to SAG/SAGA, SARAH does not require a storage of past gradients. The linear convergence rate of SARAH is proven under strong convexity assumption. We also prove a linear convergence rate (in the strongly convex case) for an inner loop of SARAH, the property that SVRG does not possess. Numerical experiments demonstrate the efficiency of our algorithm.
\end{abstract} 


\section{Introduction}
\label{introduction}
We are interested in solving a problem of the form
\begin{gather}\label{eq:problem}
\min_{w\in\R^d} \left\{ \ P(w) \eqdef \frac{1}{n} \sum_{i\in\setn} f_i(w)\right\},
\end{gather}
where each $f_i$, $i \in \setn\eqdef\{1,\dots,n\}$, is convex 
with a Lipschitz continuous gradient. Throughout the paper, we assume that there exists an optimal solution $w^{*}$ of \eqref{eq:problem}. 

Problems of this type arise frequently in supervised learning applications~\cite{ESL}. Given a training set $\{(x_i,y_i)\}_{i=1}^n$ with $x_i \in\R^{d}, y_i\in\R$, the  least squares regression model, for example, is written as \eqref{eq:problem} with $f_i(w)\eqdef (x_i^Tw-y_i)^2 + \frac{\lambda}{2} \|w\|^2$, where $\|\cdot\|$ denotes the $\ell_2$-norm. The $\ell_2$-regularized logistic regression  for binary classification is written with $f_i(w) \eqdef \log (1+\exp(-y_ix_i^Tw)) + \frac{\lambda}{2}\|w\|^2$ $(y_i\in\{-1,1\})$.

In recent years, many advanced optimization methods have been developed for problem  \eqref{eq:problem}. While the objective function is smooth and convex, the traditional optimization methods, such as gradient descent (GD) or Newton method are often impractical for this problem, when $n$ -- 
the number of training samples and hence the number of $f_i$'s -- is very large. In particular, GD updates iterates as follows
 $$w_{t+1} = w_{t} - \eta_t \nabla P(w_{t}), \quad t=0, 1, 2, \ldots .$$
Under strong convexity assumption on $P$ and with appropriate choice of $\eta_t$, GD converges at a linear rate in terms of objective function values $P(w_t)$. 
However, when $n$ is large, computing $ \nabla P(w_{t})$ at each iteration can be prohibitive. 

As an alternative, stochastic gradient descent (SGD)\footnote{We mark here that even though stochastic gradient is referred to as SG in literature, the term stochastic gradient descent (SGD) has been widely used in many important works of large-scale learning, including SAG/SAGA, SDCA, SVRG and MISO.}, originating from the seminal work of Robbins and Monro in 1951 \cite{RM1951}, has become the method of choice for solving \eqref{eq:problem}. At each step, SGD picks an index $i\in \setn$ uniformly at random, and updates the iterate as
$w_{t+1} = w_{t} - \eta_t \nabla f_i(w_{t})$, which is up-to  $n$ times cheaper than an iteration of a full gradient method. The convergence rate of SGD is 
slower than that of GD, in particular, it is sublinear in the strongly convex case. The tradeoff, however, is advantageous due to the tremendous per-iteration savings and the fact that low accuracy solutions are sufficient. This trade-off has been thoroughly analyzed in \cite{Bottou1998}.  Unfortunately, in practice SGD method is often too slow and its performance is too sensitive  to the variance in the sample gradients $\nabla f_i(w_{t})$. Use of mini-batches (averaging multiple  sample gradients $\nabla f_i(w_{t})$) was used in  \cite{pegasosICML, acceleratedmb, takac2013ICML} to reduce the variance and improve convergence rate by  constant factors. Using diminishing sequence $\{\eta_t\}$ is used to control the variance \cite{pegasos, bottou2016optimization}, but the practical convergence of SGD is known to be very sensitive to the choice of this sequence, which needs to be hand-picked.

Recently, a class of more sophisticated
algorithms have emerged, which use the specific finite-sum form of \eqref{eq:problem} and combine some deterministic and stochastic aspects to reduce variance of the steps.  The examples of these methods  are 
SAG/SAGA \cite{SAG, SAGA}, SDCA \cite{SDCA}, SVRG \cite{SVRG, Xiao2014}, DIAG \cite{mokhtari2017double}, MISO \cite{mairal2013optimization}
and S2GD \cite{S2GD}, all of which enjoy  faster convergence rate than that of SGD and use a fixed learning rate parameter $\eta$. In this paper we introduce a new method in this category, SARAH, which further improves several aspects of the existing methods. 
In Table~\ref{table:summary} we summarize complexity and some other properties of the existing methods and SARAH when applied to  strongly convex problems.
Although SVRG and SARAH have the same convergence rate, we introduce a practical variant of SARAH that outperforms SVRG in our experiments. 
 \begin{table}
 \small
\caption{Comparisons between different algorithms for strongly convex functions. $\kappa = L/\mu$ is the condition number.
}
\label{table:summary}
\centering
\begin{tabular}{C{1.25cm} c C{1.4cm} C{1.3cm} }
Method & Complexity  & Fixed Learning Rate  & Low Storage Cost \\
\hline
GD & $\Ocal\left(n\kappa \log\left({1}/{\epsilon}\right)\right)$ & \ding{51} & \ding{51}  \\
\hline
SGD & $\Ocal\left({1}/{\epsilon}\right)$  & \ding{55}  & \ding{51} \\
\hline
SVRG & $\Ocal\left((n+\kappa) \log\left({1}/{\epsilon}\right)\right)$   & \ding{51} & \ding{51} \\
\hline
SAG/SAGA & $\Ocal\left((n + \kappa) \log\left({1}/{\epsilon}\right)\right)$   & \ding{51} & \ding{55} \\
\hline
 {\textbf{SARAH}} &  {$\Ocal\left((n + \kappa) \log\left({1}/{\epsilon}\right)\right)$} & {\ding{51}} &  {\ding{51}} \\
\hline
\end{tabular}

\vskip0.3cm
\caption{Comparisons between different algorithms for convex functions.}
\label{table:summary_convex}
 
\begin{tabular}{C{3cm} c }
Method & Complexity  \\
\hline
GD & $\Ocal\left(n/\epsilon \right)$  \\
\hline
SGD & $\Ocal\left({1}/{\epsilon^2}\right)$  \\
\hline
SVRG & $\Ocal\left(n + (\sqrt{n}/\epsilon) \right)$   \\
\hline
SAGA & $\Ocal\left(n + (n/\epsilon) \right)$  \\
\hline
{\textbf{SARAH}} & {$\Ocal\left((n + (1/\epsilon)) \log(1/\epsilon) \right)$} \\
\hline
 {\textbf{SARAH (one outer loop)}} &  {$\Ocal\left(n + (1/\epsilon^2) \right)$} \\
\hline
\end{tabular}
\end{table}

In addition, theoretical results for complexity of the methods or their variants when applied to general convex functions have been derived~\cite{SAGjournal, SAGA, nonconvexSVRG, SVRG++, Katyusha}. In Table~\ref{table:summary_convex}  we summarize the key complexity results, noting that  convergence rate is now sublinear.

\paragraph{Our Contributions.} In this paper, we propose a novel algorithm which  combines some of the good properties of existing algorithms, such as SAGA and SVRG, while aiming to improve on both of these methods. In particular, our algorithm does not take steps along a stochastic gradient direction, but rather along an accumulated direction using past stochastic gradient information (as in SAGA) and occasional exact gradient information (as in SVRG).  We summarize the key properties of the proposed algorithm below. 
\begin{itemize}[noitemsep,nolistsep]
\item Similarly to SVRG, SARAH's iterations are divided into the outer loop where a full gradient is computed and the inner loop where only stochastic gradient is computed.  Unlike the case of SVRG, the steps of the inner loop of SARAH are based on accumulated stochastic information. 
\item Like SAG/SAGA and SVRG, SARAH has a sublinear rate of convergence  for general convex functions, and a linear rate of convergence for strongly convex functions.
\item SARAH uses a constant learning rate, whose size is larger than that of SVRG.  We analyze and discuss the optimal choice of the learning rate and the number of inner loop steps. However, unlike SAG/SAGA but similar to SVRG, SARAH does not require a storage of $n$ past stochastic gradients.
\item We also prove a linear convergence rate (in the strongly convex case) for the inner loop of SARAH, the property that SVRG does not possess. We show that the variance of the steps inside the inner loop goes to zero, thus SARAH is theoretically more stable and reliable than SVRG.
\item We provide a practical variant of SARAH  based on the convergence properties of the inner loop, where the simple stable stopping criterion for the inner loop is used (see Section \ref{sec:sarahplus} for more details). This variant shows how SARAH can be made more stable than SVRG in practice. 
\end{itemize}

%


\section{Stochastic Recursive Gradient Algorithm} 

Now we are ready to present our SARAH (Algorithm \ref{sarah}).

\begin{algorithm}
   \caption{SARAH}
   \label{sarah}
\begin{algorithmic}
   \STATE {\bfseries Parameters:} the learning rate $\eta > 0$ and the inner loop size $m$.
   \STATE {\bfseries Initialize:} $\tilde{w}_0$
   \STATE {\bfseries Iterate:}
   \FOR{$s=1,2,\dots$}
   \STATE $w_0 = \tilde{w}_{s-1}$
   \STATE $v_0 = \frac{1}{n}\sum_{i=1}^{n} \nabla f_i(w_0)$
   \STATE $w_1 = w_0 - \eta v_0$
   \STATE {\bfseries Iterate:}
   \FOR{$t=1,\dots,m-1$}
   \STATE Sample $i_{t}$ uniformly at random from $\setn$
   \STATE $v_{t} = \nabla f_{i_{t}} (w_{t}) - \nabla f_{i_{t}}(w_{t-1}) + v_{t-1}$
   \STATE $w_{t+1} = w_{t} - \eta v_{t}$
   \ENDFOR
   \STATE Set $\tilde{w}_s = w_{t}$ with $t$ chosen uniformly at random from $\{0,1,\dots,m\}$
   \ENDFOR
\end{algorithmic}
\end{algorithm} 

The key step of the algorithm is a recursive update of the stochastic gradient estimate \textit{(SARAH update)}
\begin{equation}\label{eq:vt}
   v_{t} = \nabla f_{i_{t}} (w_{t}) - \nabla f_{i_{t}}(w_{t-1}) + v_{t-1},
\end{equation}
followed by the iterate update:
\begin{equation}\label{eq:iterate}
w_{t+1} = w_{t} - \eta v_{t}.
\end{equation}
For comparison, SVRG update can be written in a similar way as
\begin{equation}\label{eq:svrgvt}
   v_{t} = \nabla f_{i_{t}} (w_{t}) - \nabla f_{i_{t}}(w_{0}) + v_{0}. 
\end{equation}

Observe that in SVRG, $v_t$ is an unbiased estimator of the gradient, while it is not true for SARAH. Specifically,  \footnote{
$\Exp [\cdot | \mathcal{F}_{t}] = \Exp_{i_{t}}[\cdot]$, which is  expectation with respect to the random choice of index $i_{t}$ (conditioned on
$w_0, i_1, i_2, \dots, i_{t-1}$).}
\begin{equation}\label{eq:NotSGD}
\Exp[v_{t} | \mathcal{F}_{t}]  = \nabla P(w_{t}) - \nabla P(w_{t-1}) + v_{t-1}\neq \nabla P(w_{t}),
\end{equation} 
where \footnote{$\mathcal{F}_{t}$ also contains all the information of $w_0,\dots,w_{t}$ as well as $v_0,\dots,v_{t-1}.$} $\mathcal{F}_{t} = \sigma(w_0,i_1,i_2,\dots,i_{t-1})$ is the $\sigma$-algebra generated by $w_0,i_1,i_2,\dots,i_{t-1}$; $\mathcal{F}_{0} = \mathcal{F}_{1} = \sigma(w_0)$.  
Hence, SARAH is different from SGD and SVRG type of methods, however, the following total expectation holds, $\Exp[v_{t}] = \Exp[\nabla P(w_{t})]$, differentiating SARAH from SAG/SAGA.

SARAH is similar to SVRG since they both contain outer loops which require one full gradient evaluation per outer iteration followed by one full gradient descent step  with a given learning rate. The difference lies in the inner loop, where SARAH updates the stochastic step direction $v_t$ recursively by adding and subtracting component gradients to and from the previous $v_{t-1}\ (t\geq 1)$ in \eqref{eq:vt}. Each inner iteration evaluates 
$2$ stochastic gradients and hence the total work per outer iteration is $\Ocal(n+m)$ in terms of the number of gradient evaluations. 
Note that due to its nature, without running the inner loop, i.e., $m = 1$, SARAH reduces to the  GD algorithm.


\section{Theoretical Analysis}

To proceed with the analysis of the proposed algorithm, we will make the following common assumptions.
\begin{ass}[$L$-smooth]
\label{ass_Lsmooth}
Each $f_i: \mathbb{R}^d \to \mathbb{R}$, $i \in \setn$, is $L$-smooth, i.e., there exists a constant $L > 0$ such that
$$
\| \nabla f_i(w) - \nabla f_i(w') \| \leq L \| w - w' \|, \ \forall w,w' \in \mathbb{R}^d.
$$
\end{ass}
 Note that this assumption implies that $P(w) = \frac{1}{n}\sum_{i=1}^n f_i(w)$ is also \emph{L-smooth}. The following strong convexity assumption
 will be made for the appropriate parts of the analysis, otherwise, it would be dropped.  
\begin{subass}
\begin{ass}[$\mu$-strongly convex]
\label{ass_stronglyconvex}
The function $P: \mathbb{R}^d \to \mathbb{R}$, is $\mu$-strongly convex, i.e., there exists a constant $\mu > 0$ such that $\forall w,w' \in \mathbb{R}^d$, 
\begin{gather*}
P(w)   \geq  P(w') + \nabla P(w')^\top (w - w') + \tfrac{\mu}{2}\|w - w'\|^2.
\end{gather*}
\end{ass}
Another, stronger, assumption of $\mu$-strong convexity for \eqref{eq:problem} will also be imposed when required in our analysis. Note that Assumption~\ref{ass_stronglyconvex2} implies Assumption~\ref{ass_stronglyconvex} but not vice versa.
\begin{ass}
\label{ass_stronglyconvex2}
Each function $f_i: \mathbb{R}^d \to \mathbb{R}$, $i \in \setn$, is strongly convex with $\mu>0$. 
\end{ass}
\end{subass}
Under Assumption \ref{ass_stronglyconvex}, let us define the (unique) optimal solution of \eqref{eq:problem} as $w^{*}$, 
Then strong convexity of $P$  implies that 
 \begin{equation}\label{eq:strongconvexity2}
 2\mu [ P(w) - P(w^{*})] \leq  \| \nabla P(w)\|^2, \ \forall w \in \mathbb{R}^d. 
 \end{equation}
We note here, for future use, that for strongly convex functions of the form \eqref{eq:problem}, arising in machine learning applications, the condition number is defined as $\kappa\eqdef L/\mu$. Furthermore, we should also notice that Assumptions \ref{ass_stronglyconvex} and \ref{ass_stronglyconvex2} both cover a wide range of problems, e.g. $l_2$-regularized empirical risk minimization problems with convex losses. 


Finally, as a special case of the strong convexity of all $f_i$'s with $\mu=0$, we state the general convexity assumption, which we will use for convergence analysis.
\begin{ass}
\label{ass_convex}
Each function $f_i: \mathbb{R}^d \to \mathbb{R}$, $i \in \setn$, is convex, i.e.,
\begin{gather*}
f_i(w)    \geq f_i(w') + \nabla f_i(w')^\top (w - w'), \quad \forall i\in\setn. 
\end{gather*}
\end{ass}
Again, we note that Assumption~\ref{ass_stronglyconvex2} implies Assumption~\ref{ass_convex}, but Assumption  \ref{ass_stronglyconvex} does not.
Hence in our analysis, depending on the result we aim at, we will require Assumption~\ref{ass_convex} to hold by itself, or Assumption  \ref{ass_stronglyconvex}  and Assumption~\ref{ass_convex} to hold together, or Assumption~\ref{ass_stronglyconvex2} to hold by itself. We will always use Assumption  \ref{ass_Lsmooth}.

Our iteration complexity analysis aims to bound the number of outer iterations $\mathcal{T}$ (or total number of stochastic gradient evaluations) which is needed to  guarantee that 
$\|\nabla P(w_\mathcal{T})\|^2\leq \epsilon$. In this case we will say that $w_\mathcal{T}$ is an $\epsilon$-accurate solution. 
However, as is  common practice for  
stochastic gradient
 algorithms, we aim to obtain the bound on the number of iterations, which is required to guarantee the bound on the expected squared norm of a gradient, i.e.,
\begin{equation}\label{eq:accuracy}
\Exp [\| \nabla P(w_\mathcal{T}) \|^2] \leq \epsilon.
\end{equation}

\subsection{Linearly Diminishing Step-Size  in a Single Inner Loop}\label{sec:linearconvergence}

The most important property of the SVRG algorithm is  the variance reduction of the steps. This property holds as the number of outer iteration grows, but it does not hold, if only the number of inner iterations increases.  In other words, if we simply run the inner loop for many iterations (without executing additional outer loops), the variance of the steps does not reduce in the case of SVRG, while it goes to zero in the case of SARAH. 
To illustrate this effect, let us take a look at Figures~\ref{fig:VR1} and \ref{fig:VR2}. 

In Figure~\ref{fig:VR1}, we applied one outer loop of SVRG and SARAH to a sum of $5$  quadratic functions in a two-dimensional space, where the optimal solution is at the origin, the black lines and black dots indicate the trajectory of each algorithm and the red point indicates the final iterate. 
Initially, both SVRG and SARAH take steps along stochastic gradient directions towards the optimal solution. However, later iterations of SVRG wander randomly around the origin with large deviation from it, while SARAH follows a much more stable convergent trajectory, with a final iterate falling in a small neighborhood of the optimal solution. 

\begin{figure}  
\centering
 \epsfig{file=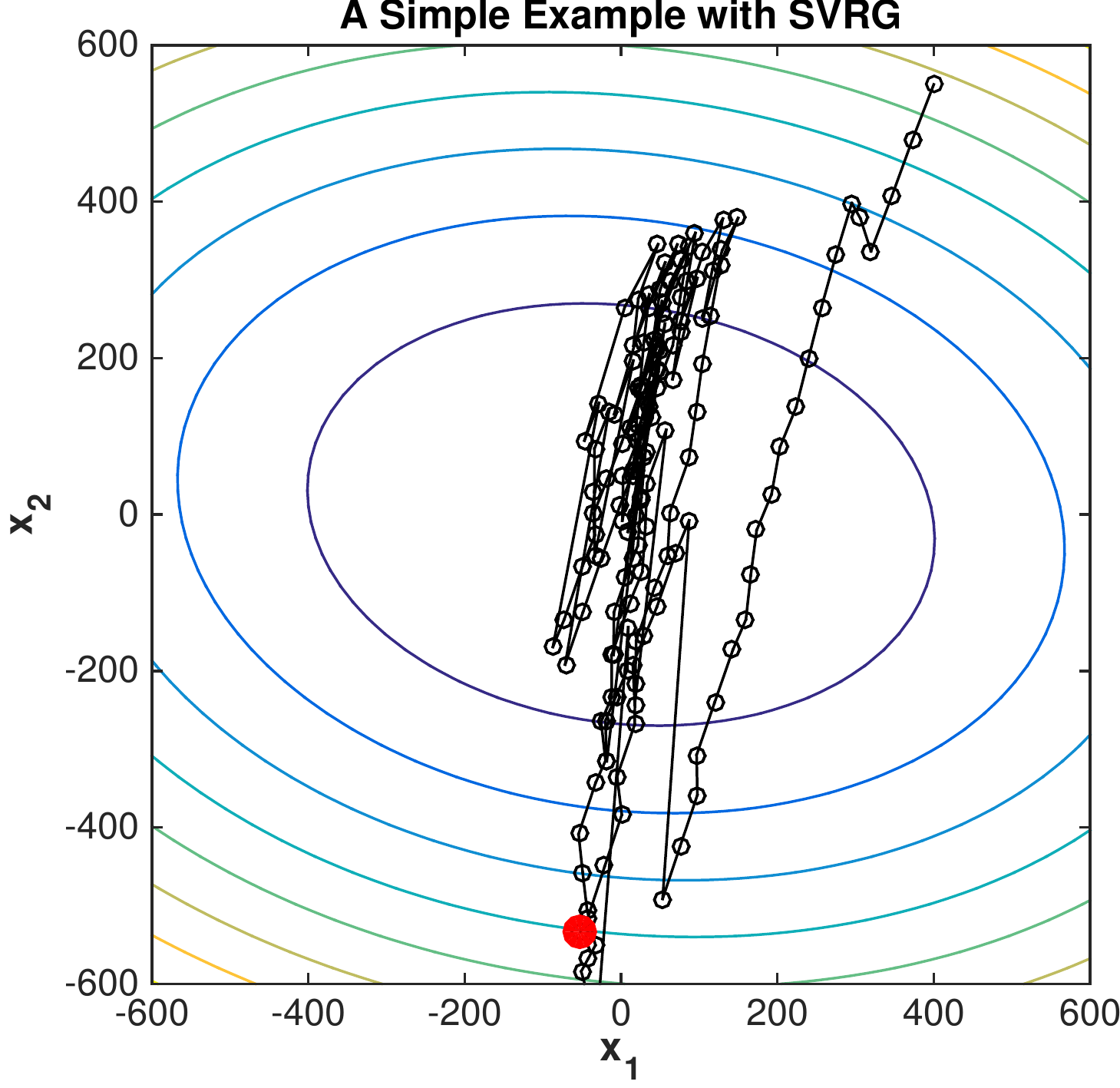,width=0.22\textwidth}
 \hspace{4mm}
 \epsfig{file=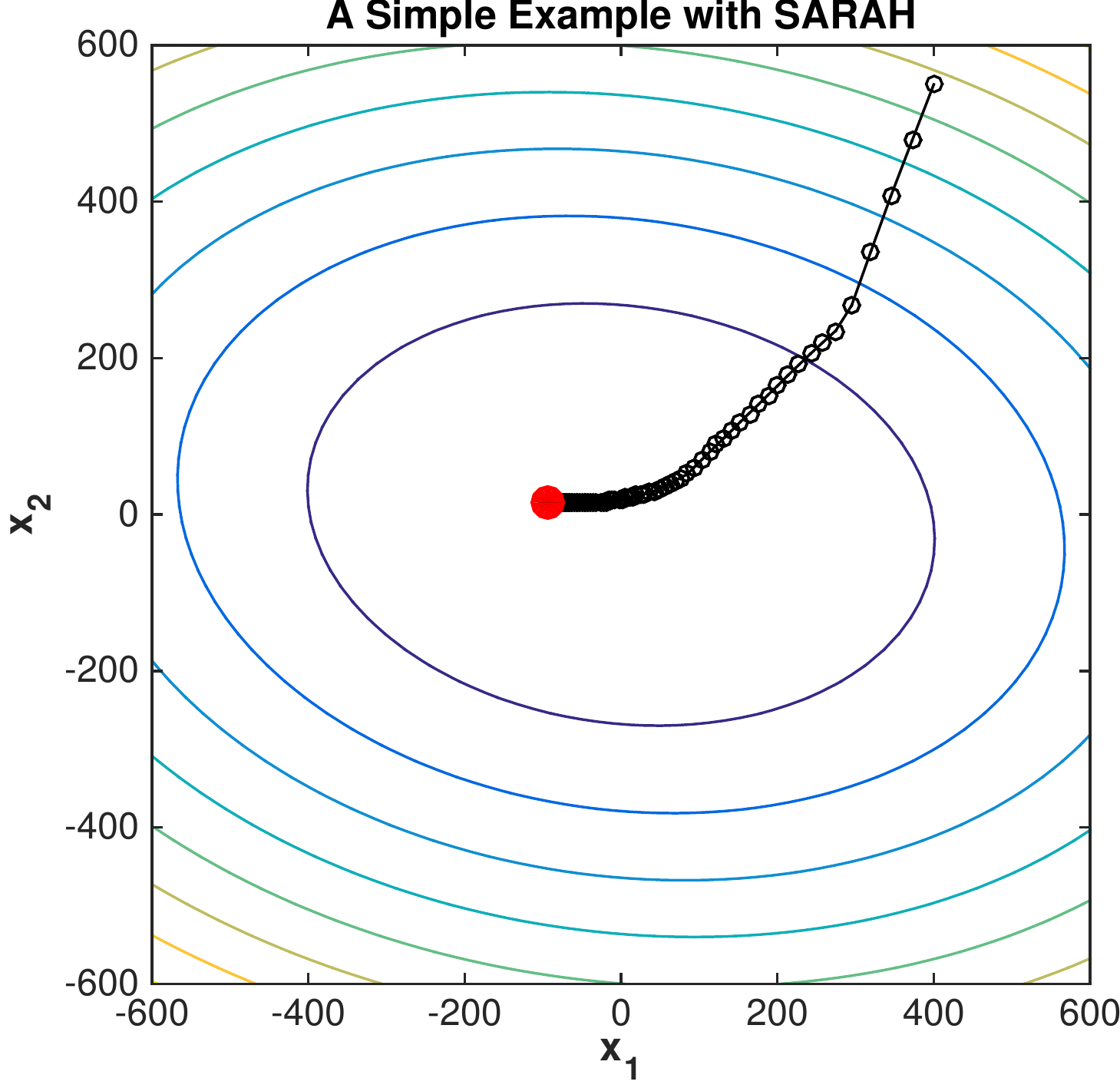,width=0.22\textwidth}
  \caption{\footnotesize A two-dimensional example of $\min_w P(w)$ with $n=5$ for SVRG (left) and SARAH (right).}
  \label{fig:VR1}
  
$\ $\\$\ $\\ 
\centering
 \epsfig{file=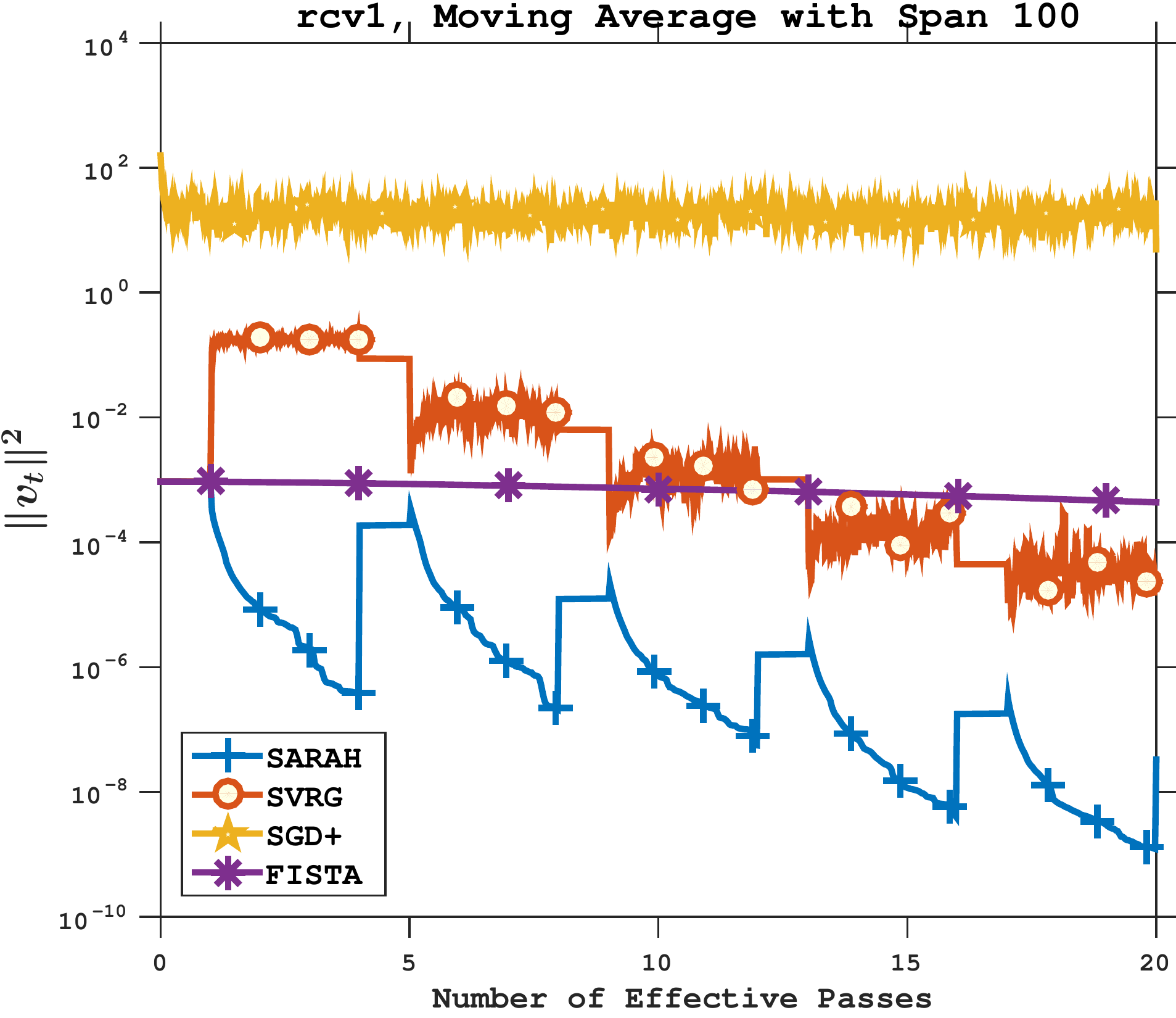,width=0.22\textwidth}
 \hspace{4mm}
 \epsfig{file=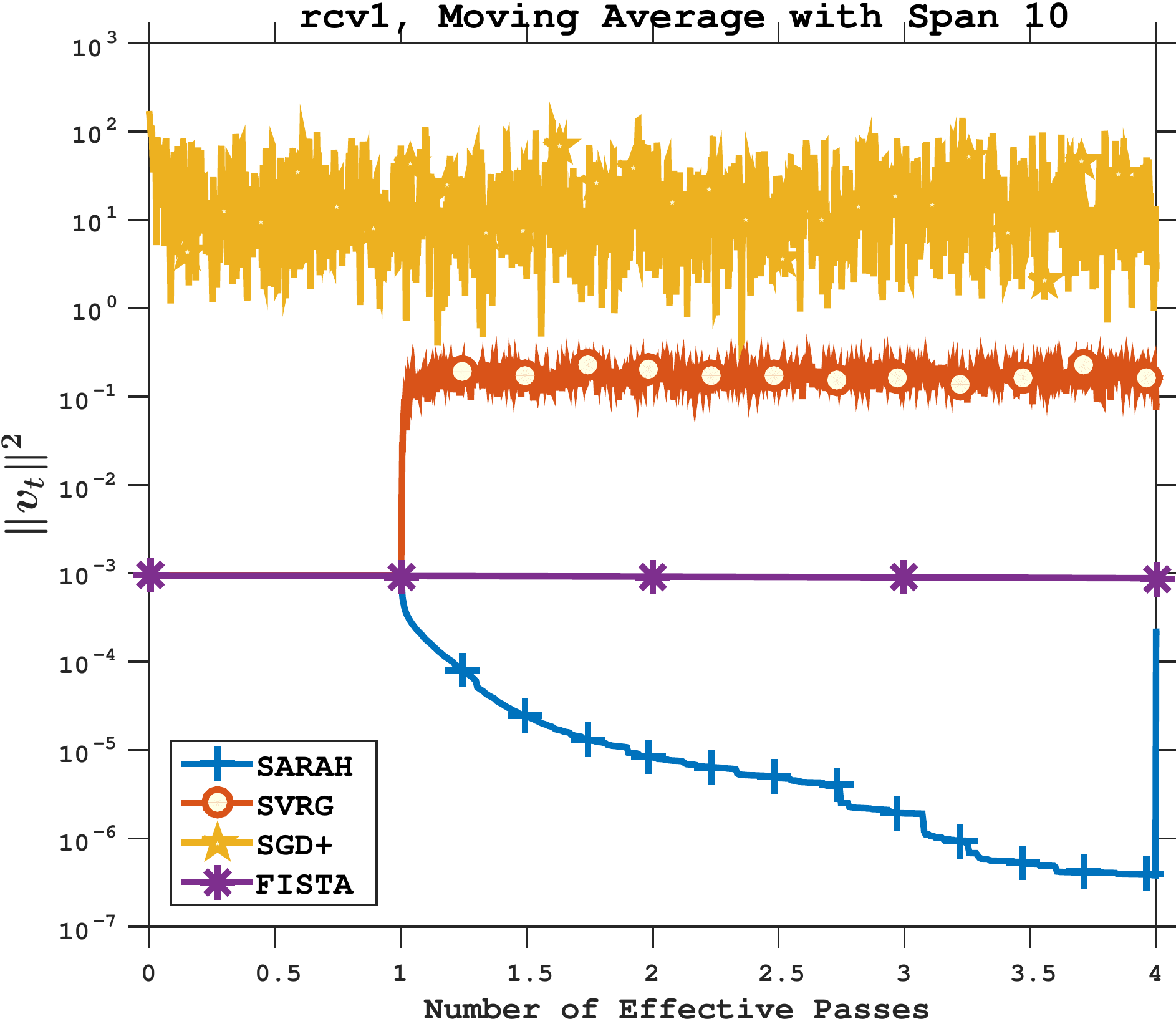,width=0.22\textwidth}
  \caption{\footnotesize An example of $\ell_2$-regularized logistic regression on \emph{rcv1} training dataset for SARAH, SVRG, SGD+ and FISTA with multiple outer iterations (left) and a single outer iteration (right).}
  \label{fig:VR2}
 \end{figure}

In Figure~\ref{fig:VR2}, the x-axis  denotes the  \emph{number of effective passes} which is equivalent to the  number of passes through all of the data in the dataset, the cost of each pass being  equal to the cost of one full gradient evaluation; and y-axis represents $\|v_t\|^2$.
Figure~\ref{fig:VR2} shows the evolution of  $\|v_t\|^2$ for SARAH, SVRG, SGD+ (SGD with decreasing learning rate) and FISTA (an accelerated version of  GD~\cite{fista}) with $m=4n$, 
where the left plot shows the trend over multiple outer iterations and the right plot shows a single outer iteration\footnote{In the plots of Figure~\ref{fig:VR2}, since the data for SVRG is noisy, we smooth it by using moving average filters with spans 100 for the left plot and 10 for the right one.}. We can see that for SVRG, $\|v_t\|^2$  decreases over the outer iterations, while it has an increasing trend or oscillating trend for each inner loop. In contrast, SARAH enjoys decreasing trends both in the outer  and the inner loop iterations.

We will now show that the stochastic steps computed by SARAH converge linearly in the inner loop.  We present two linear convergence results  based on our two different assumptions of $\mu$-\emph{strong convexity}. These results substantiate our conclusion that SARAH  uses more stable stochastic gradient estimates than SVRG. The following theorem is our first result to demonstrate the linear convergence of our stochastic recursive step $v_t$.

\begin{subthm}
\begin{thm}\label{lem_bouned_moment_stronglyconvexP}
Suppose that Assumptions \ref{ass_Lsmooth}, \ref{ass_stronglyconvex} and \ref{ass_convex} hold. Consider $v_{t}$ defined by \eqref{eq:vt} in SARAH (Algorithm \ref{sarah}) with $\eta < 2/L$. Then, for any $t\geq 1$,
\begin{align*}
\mathbb{E}[\|v_{t}\|^2]
 &\leq \left[ 1 - \left(\tfrac{2}{\eta L} - 1 \right) \mu^2 \eta^2  \right] \mathbb{E}[\|v_{t-1}\|^2]
 \\
 &\leq \left[ 1 - \left(\tfrac{2}{\eta L} - 1 \right) \mu^2 \eta^2  \right]^{t} \mathbb{E}[\| \nabla P(w_{0}) \|^2].
\end{align*}
\end{thm}
This result implies that by choosing $\eta=\Ocal(1/L)$, we obtain the linear  convergence of $\|v_t\|^2$ in expectation with the rate $(1-1/\kappa^2)$. Below we show that
 a better convergence rate can be obtained under a  stronger convexity assumption. 

\begin{thm}\label{thm:bound_moment}
Suppose that Assumptions \ref{ass_Lsmooth} and \ref{ass_stronglyconvex2} hold.  Consider $v_{t}$ defined by \eqref{eq:vt} in SARAH (Algorithm \ref{sarah}) with $\eta \leq 2/(\mu+L)$. Then the following bound holds, $\forall\ t\geq 1$, 
\begin{align*}
\mathbb{E}[\| v_{t} \|^2 ]
& \leq \left( 1 - \tfrac{2 \mu L \eta}{\mu + L} \right) \Exp[  \|v_{t-1} \|^2 ]
\\
& \leq \left(1 - \tfrac{2\mu L \eta}{\mu + L} \right)^{t} \Exp[ \| \nabla P(w_{0}) \|^2]. 
\end{align*}
\end{thm}

\end{subthm}

Again, by setting $\eta=\Ocal(1/L)$, we derive the linear convergence with the rate of $(1-1/\kappa)$, which is a significant improvement over the result of Theorem~\ref{lem_bouned_moment_stronglyconvexP}, when the problem is severely ill-conditioned.

\subsection{Convergence Analysis}


In this section, we derive the general convergence rate results for Algorithm \ref{sarah}. First, we present two important Lemmas as the foundation of our theory. Then, we proceed to prove sublinear convergence rate of a single outer iteration when applied to general convex functions. In the end, we prove that the algorithm with multiple outer iterations has linear convergence rate in the strongly convex case.

We begin with proving two useful lemmas that do not require any convexity assumption.
The first Lemma 
\ref{lem_main_derivation}
bounds the  sum of expected values of
$\|\nabla P(w_t)\|^2$.
The second, Lemma \ref{lem:var_diff_01},
bounds $\mathbb{E}[ \| \nabla P(w_{t}) - v_{t} \|^2 ]$. 

\begin{lem}\label{lem_main_derivation}
Suppose that Assumption \ref{ass_Lsmooth} holds. Consider SARAH (Algorithm \ref{sarah}). Then, we have 
\begin{align*}
& \sum_{t=0}^{m} \mathbb{E}[ \| \nabla P(w_{t})\|^2 ]  \leq \frac{2}{\eta} \mathbb{E}[ P(w_{0}) - P(w^{*})]
\tagthis \label{eq:001} 
\\&\quad+ \sum_{t=0}^{m} \mathbb{E}[ \| \nabla P(w_{t}) - v_{t} \|^2 ]  
 - ( 1 - L\eta ) \sum_{t=0}^{m} \mathbb{E} [ \| v_{t} \|^2 ].
\end{align*}
\end{lem}
\begin{lem}\label{lem:var_diff_01}
Suppose that Assumption \ref{ass_Lsmooth} holds. Consider $v_{t}$ defined by \eqref{eq:vt} in SARAH (Algorithm \ref{sarah}). Then for any $t\geq 1$, 
\begin{align*}
&\mathbb{E}[ \| \nabla P(w_{t}) - v_{t} \|^2 ] 
= \sum_{j = 1}^{t} \mathbb{E}[ \| v_{j} - v_{j-1} \|^2 ]  
 \\&\qquad\qquad
 - \sum_{j = 1}^{t} \mathbb{E}[ \| \nabla P(w_{j}) - \nabla P(w_{j-1}) \|^2 ]. 
\end{align*}
\end{lem}
Now we are ready to provide our main theoretical results. 

\subsubsection{General Convex Case}

Following from Lemma \ref{lem:var_diff_01}, we can obtain the following upper bound for $\mathbb{E}[ \| \nabla P(w_{t}) - v_{t} \|^2 ]$ for convex functions $f_i, i\in\setn$.

\begin{lem}\label{lem_bound_var_diff_str_02}
Suppose that Assumptions \ref{ass_Lsmooth} and \ref{ass_convex} hold. Consider $v_{t}$ defined as \eqref{eq:vt} in SARAH (Algorithm \ref{sarah}) with $\eta < 2/L$. Then we have that for any $t\geq 1$, 
\begin{align*}
\mathbb{E}[ \| \nabla P(w_{t}) - v_{t} \|^2 ] 
&\leq  \frac{\eta L}{2 - \eta L} \Big[ \mathbb{E}[ \|v_{0} \|^2] - \mathbb{E}[\| v_{t} \|^2] \Big]
\\
&\leq  \frac{\eta L}{2 - \eta L}  \mathbb{E}[ \|v_{0} \|^2].
\tagthis\label{eq:bound1}
\end{align*}
\end{lem}

%

Using  the above  lemmas, we can state and prove one of our core theorems as follows.
\begin{thm}\label{thm:generalconvex_01}
Suppose that Assumptions \ref{ass_Lsmooth} and \ref{ass_convex} hold. Consider  SARAH (Algorithm \ref{sarah}) with $\eta \leq 1/L$. Then  for any $s\geq 1$, we have 
\begin{align*}
\mathbb{E}[ \| \nabla P(\tilde{w}_s)\|^2 ] 
&\leq \frac{2}{\eta (m + 1)} \mathbb{E}[ P(\tilde w_{s-1}) - P(w^{*})]  
\\ 
&\qquad + \frac{ \eta L}{2 - \eta L}  \mathbb{E}[ \| \nabla P(\tilde w_{s-1})\|^2 ]. \tagthis \label{eq:agasgsasw}
\end{align*}
\end{thm}
\begin{proof}
Since $v_0=\nabla P(w_0)$ implies $\| \nabla P(w_{0}) - v_{0} \|^2 = 0$ then by Lemma \ref{lem_bound_var_diff_str_02}, we can write
\begin{align*}
& \textstyle{\sum}_{t=0}^{m} \mathbb{E}[ \| \nabla P(w_{t}) - v_{t} \|^2 ] 
 \leq  \frac{m\eta L}{2 - \eta L} \mathbb{E}[ \|v_{0} \|^2]. \tagthis \label{eq:abcdef}
\end{align*}
Hence, by Lemma \ref{lem_main_derivation} with $\eta\leq 1/L$, we have
\begin{align*}
& \textstyle{\sum}_{t=0}^{m} \mathbb{E}[ \| \nabla P(w_{t})\|^2 ]  \\
 & \leq \tfrac{2}{\eta} \mathbb{E}[ P(w_{0}) - P(w^{*})] + \textstyle{\sum}_{t=0}^{m} \mathbb{E}[ \| \nabla P(w_{t}) - v_{t} \|^2 ] 
 \\
& \overset{\eqref{eq:abcdef}}{\leq} \tfrac{2}{\eta} \mathbb{E}[ P(w_{0}) - P(w^{*})]  + \tfrac{m\eta L}{2 - \eta L} \mathbb{E}[ \| v_{0} \|^2 ]. \tagthis\label{eq:thm1conv}
\end{align*}
Since we are considering one outer iteration, with $s\geq 1$, then we have $v_0 = \nabla P(w_0) = \nabla P(\tilde w_{s-1})$ (since $w_0 =\tilde{w}_{s-1}$), and $\tilde{w}_s = w_t$, where $t$ is picked uniformly at random from $\{0,1,\dots,m\}$. Therefore, the following holds,
\begin{align*}
\mathbb{E}[ \| \nabla P(\tilde{w}_s)\|^2 ] &= \tfrac{1}{m+1}\textstyle{\sum}_{t=0}^{m} \mathbb{E}[ \| \nabla P(w_{t})\|^2 ]
 \\
&\overset{\eqref{eq:thm1conv}}{\leq} \tfrac{2}{\eta (m + 1)} \mathbb{E}[ P(\tilde w_{s-1}) - P(w^{*})]  
\\ 
&\qquad + \tfrac{ \eta L}{2 - \eta L}  \mathbb{E}[ \| \nabla P(\tilde w_{s-1})\|^2 ]. \qedhere 
\end{align*}
\end{proof}
 Theorem~\ref{thm:generalconvex_01}, in the case when $\eta\leq 1/L$ implies that
\begin{align*}
\mathbb{E}[ \| \nabla P(\tilde{w}_s)\|^2 ] 
&\leq \tfrac{2}{\eta (m + 1)} \mathbb{E}[ P(\tilde w_{s-1}) - P(w^{*})]  
\\ 
&\qquad + { \eta L} \mathbb{E}[ \| \nabla P(\tilde w_{s-1})\|^2].
\end{align*}
By choosing the learning rate $\eta = \sqrt{\frac{2}{L(m+1)}}$ (with $m$ such that   $\sqrt{\frac{2}{L(m+1)}}\leq 1/L$) we can derive the following convergence result,
\begin{align*}
&\mathbb{E}[ \| \nabla P(\tilde{w}_s)\|^2 ] 
\\ 
&\qquad \leq \sqrt{\tfrac{2L}{m + 1}} \mathbb{E}[ P(\tilde w_{s-1}) - P(w^{*}) + \| \nabla P(\tilde w_{s-1})\|^2].
\end{align*}
Clearly, this result shows a sublinear convergence rate for SARAH under general convexity assumption within a single inner loop, with increasing $m$, 
and consequently, we have the following result  for complexity bound. 
\begin{cor}\label{cor:generalconvex_1}
Suppose that Assumptions \ref{ass_Lsmooth} and \ref{ass_convex} hold. Consider SARAH (Algorithm \ref{sarah}) within a single outer iteration with the learning rate $\eta = \sqrt{\frac{2}{L(m+1)}}$ where $m\geq 2L - 1$ is the total number of iterations, then  $\|\nabla P(w_t)\|^2$ converges
sublinearly in expectation  with a rate of $\sqrt{\frac{2L}{m+1}}$, and therefore, the total complexity to achieve an $\epsilon$-accurate solution defined in \eqref{eq:accuracy} is $\Ocal(n+1/\epsilon^2)$.   
\end{cor}

\begin{figure*}
\centering
 \epsfig{file=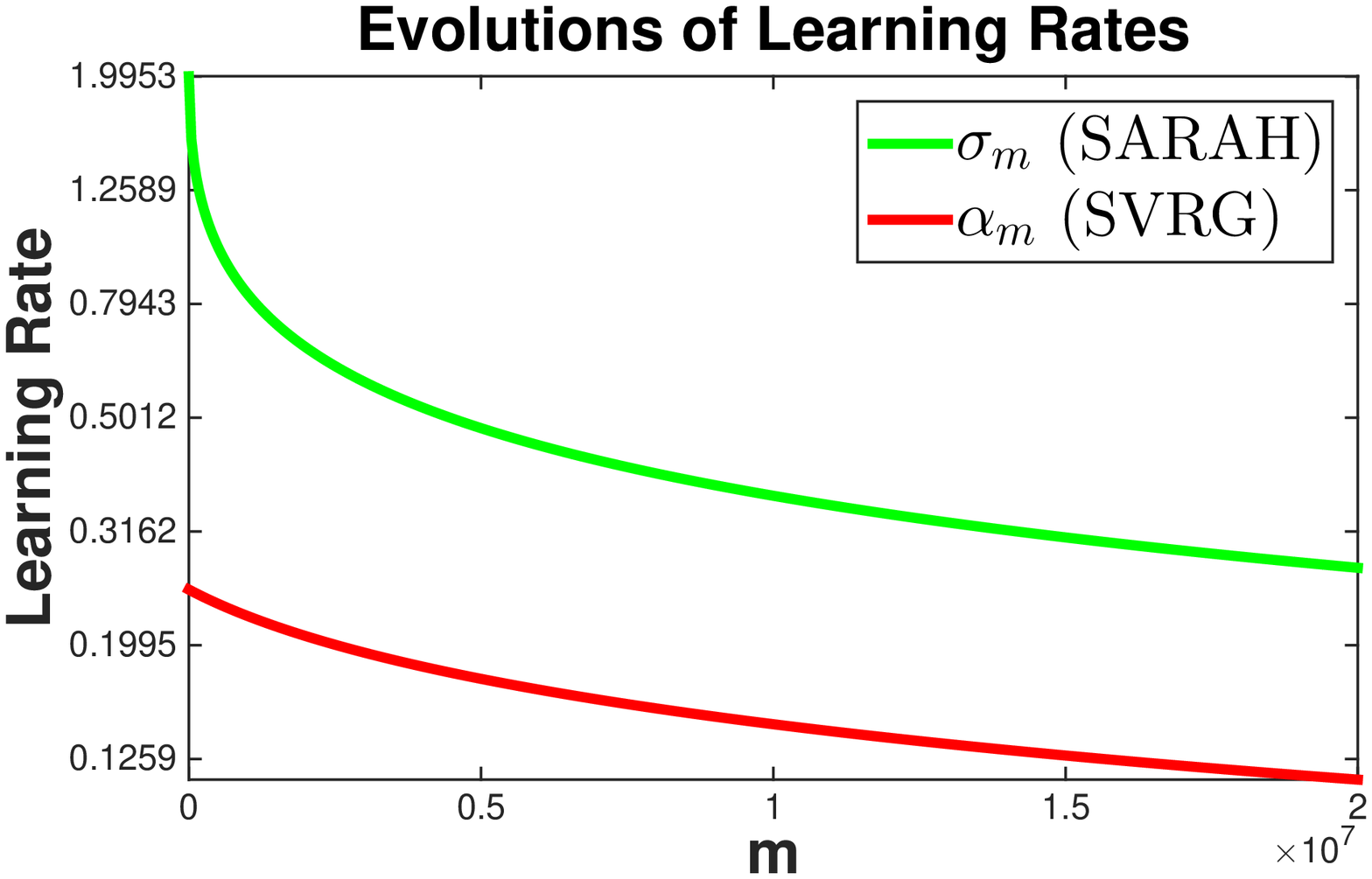,width=0.3\textwidth} 	 
 \epsfig{file=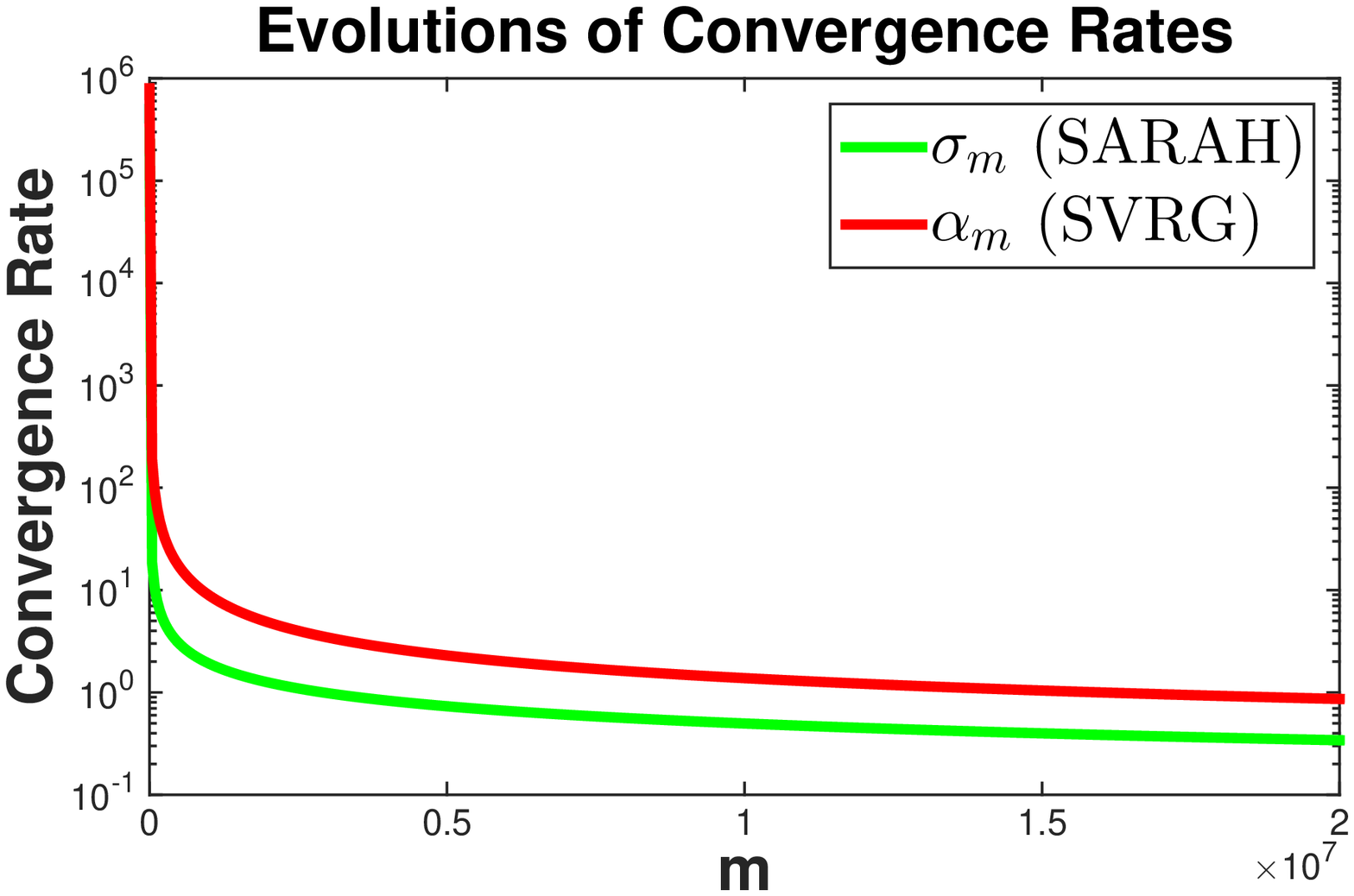,width=0.3\textwidth}
 \epsfig{file=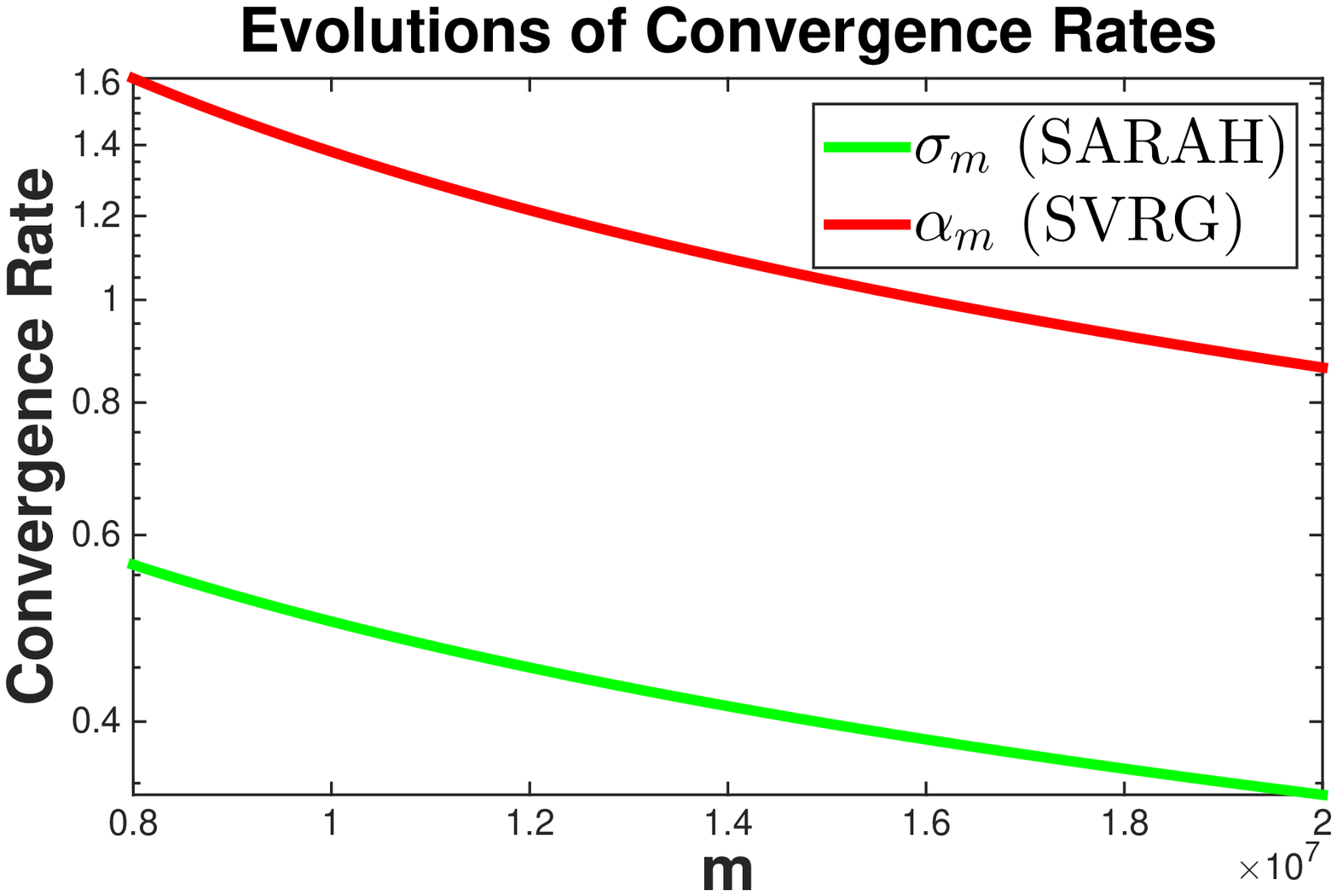,width=0.3\textwidth}
  \caption{\footnotesize Theoretical comparisons of learning rates (left) and convergence rates (middle and right) with $n=1,000,000$ for SVRG and SARAH in one inner loop.}
  \label{fig:comparison}
 \end{figure*}

We now turn to estimating convergence of SARAH with multiple outer steps. Simply using  Theorem~\ref{thm:generalconvex_01} for each of the outer steps we have the following result.
\begin{thm}\label{thm:generalconvex_02}
Suppose that Assumptions \ref{ass_Lsmooth} and \ref{ass_convex} hold. Consider SARAH (Algorithm \ref{sarah}) and define 
\begin{align*}
\delta_{k} &= \tfrac{2}{\eta(m+1)} \mathbb{E}[ P(\tilde{w}_{k}) - P(w^{*})], \ k = 0,1,\dots,s-1, 
\end{align*}
and $\delta = \max_{0 \leq k \leq s-1} \delta_{k}$. Then we have
\begin{align*}
\tagthis\label{asdfsfas}
\mathbb{E}[ \| \nabla P(\tilde{w}_s)\|^2 ] - \Delta  \leq \alpha^s ( \| \nabla P(\tilde{w}_0)\|^2 - \Delta),  
\end{align*}
where
$\Delta  = \delta\left(1 + \tfrac{\eta L}{2(1 - \eta L)} \right),$ and
$\alpha  = \tfrac{\eta L}{2 - \eta L}.$
\end{thm}
Based on Theorem \ref{thm:generalconvex_02}, we have the following total complexity for SARAH in the general convex case. 
\begin{cor}\label{cor:thm_3_complexity}
 Let us choose $\Delta = \epsilon/4$, $\alpha=1/2$ (with $\eta = 2/(3 L)$), and $m = \Ocal(1/\epsilon)$ in Theorem \ref{thm:generalconvex_02}. Then, the total complexity to achieve an $\epsilon$-accuracy solution defined in \eqref{eq:accuracy} is $\Ocal((n + (1/\epsilon))\log(1/\epsilon))$. 
\end{cor}

\subsubsection{Strongly Convex Case}

We now turn to 
the discussion of the linear convergence rate of SARAH under the strong convexity assumption on $P$.
From Theorem~\ref{thm:generalconvex_01}, for any $s\geq 1$, using property \eqref{eq:strongconvexity2} of the $\mu$-\emph{strongly convex} $P$, we have
\begin{align*}
&\mathbb{E}[ \| \nabla P(\tilde{w}_s)\|^2 ] 
\leq \tfrac{2}{\eta (m + 1)} \mathbb{E}[ P(\tilde w_{s-1}) - P(w^{*})]  
\\ 
&\qquad\qquad\qquad\qquad + \tfrac{ \eta L}{2 - \eta L}  \mathbb{E}[ \| \nabla P(\tilde w_{s-1})\|^2 ] 
\\ 
&\qquad \overset{\eqref{eq:strongconvexity2}}{\leq} 
\left( \tfrac{1}{\mu \eta (m+1)} + \tfrac{\eta L}{2 - \eta L} \right) \mathbb{E}[ \| \nabla P(\tilde w_{s-1})\|^2 ],
\end{align*}
and equivalently,
\begin{equation}\label{eq:recursive}
\Exp [ \| \nabla P(\tilde{w}_s)\|^2 ] \leq \sigma_m \Exp [ \| \nabla P(\tilde w_{s-1})\|^2 ].
\end{equation}
Let us define $\sigma_m \eqdef \tfrac{1}{\mu \eta (m + 1)} +  \tfrac{\eta L}{2 - \eta L}$.  
Then by choosing  $\eta$ and $m$ such that $\sigma_m<1$, and applying \eqref{eq:recursive} recursively, we are able to reach the following convergence result.
\begin{thm}\label{thm:stronglyconvexconvergence}
Suppose that Assumptions \ref{ass_Lsmooth}, \ref{ass_stronglyconvex} and \ref{ass_convex} hold. Consider  SARAH (Algorithm \ref{sarah})
 with the choice of $\eta$ and $m$ such that
\begin{equation}\label{eq:sigma0}
\sigma_m \eqdef \frac{1}{\mu \eta (m + 1)} +  \frac{\eta L}{2 - \eta L}  < 1. 
\end{equation}
Then, we have
\begin{align*}
\mathbb{E}[ \| \nabla P(\tilde{w}_s)\|^2 ] \leq (\sigma_m)^s \| \nabla P(\tilde{w}_0)\|^2. 
\end{align*}
\end{thm}
\begin{remark}\label{rem1}
Theorem~\ref{thm:stronglyconvexconvergence} implies that any $\eta<1/L$ will work for SARAH. Let us compare our convergence rate to that of  SVRG. The linear rate of SVRG,  as presented in \cite{SVRG}, is given by
 \begin{align*}
\alpha_m = \tfrac{1}{\mu \eta (1 - 2 L \eta) m} + \tfrac{2 \eta L}{1 - 2 \eta L} < 1. 
\end{align*}
We  observe that it implies that the learning rate  has to satisfy $\eta < 1/(4L)$, which is a tighter restriction than $\eta<1/L$ required by SARAH. In addition, with the same values of $m$ and $\eta$, the rate or convergence of (the outer iterations) of SARAH is always smaller than that of SVRG. 
\begin{align*}
\sigma_m &= \tfrac{1}{\mu \eta (m + 1)} +  \tfrac{\eta L}{2 - \eta L} = \tfrac{1}{\mu \eta (m + 1)} +  \tfrac{1}{2/(\eta L) - 1} 
 \\
&< \tfrac{1}{\mu \eta (1 - 2 L \eta) m} + \tfrac{1}{0.5/(\eta L) - 1}  = \alpha_m. 
\end{align*}
\end{remark}
 \begin{remark}\label{rem2}
 To further demonstrate the better  convergence properties of SARAH, let us consider following optimization problem
 $$\min_{0<\eta<1/L}\ \sigma_m,
 \qquad \min_{0<\eta<1/{4L}}\ \alpha_m,$$
 which can be interpreted as the best convergence rates for different values of $m$, for both SARAH and SVRG. After simple calculations, we 
  plot both learning rates and the corresponding theoretical rates of convergence, as shown in Figure~\ref{fig:comparison},
   where the right plot is a zoom-in on a part of the middle plot. The left plot shows that the  optimal learning rate for SARAH is significantly larger than that of  SVRG, while the other two plots show significant improvement upon outer iteration convergence rates for SARAH over SVRG.
  \end{remark}
Based on Theorem~\ref{thm:stronglyconvexconvergence}, we are able to derive the following total complexity for SARAH in the strongly convex case.
\begin{cor}\label{cor:complexity}
Fix $\epsilon\in(0,1)$, and let us run SARAH with $\eta = 1/(2L)$ and $m = 4.5\kappa$ for $\mathcal{T}$ iterations where
$\mathcal{T} = \lceil \log(\|\nabla P(\tilde{w}_0)\|^2/\epsilon)/\log(9/7) \rceil,$
then we can derive an $\epsilon$-accuracy solution defined in \eqref{eq:accuracy}. Furthermore, we can obtain the total complexity of SARAH, to achieve the $\epsilon$-accuracy solution, as
$\Ocal\left((n+\kappa)\log(1/\epsilon)\right).$
\end{cor}

\section{A Practical Variant}
\label{sec:sarahplus}
While SVRG is an efficient variance-reducing stochastic gradient method, one of its main drawbacks is the sensitivity of the practical performance with respect to the choice of $m$. It is know that $m$ should be around $\Ocal(\kappa)$,\footnote{
In practice, when $n$ is large, $P(w)$ is often considered as a regularized Empirical Loss Minimization problem with regularization parameter $\lambda = \frac1n$, then $\kappa \sim \Ocal(n).$
}
 while it still remains unknown that what the exact best choice is. In this section, we propose a practical variant of SARAH as SARAH+ (Algorithm \ref{sarah_sc}), which provides an automatic and adaptive  choice of the inner loop size $m$. Guided by the linear convergence of the steps in the inner loop, demonstrated in Figure~\ref{fig:VR2}, we introduce a stopping criterion based on the values of $\|v_t\|^2$ while upper-bounding the total number of steps by a large enough $m$ for robustness. The other modification compared to SARAH (Algorithm \ref{sarah}) is the more practical choice $\tilde{w}_s = w_{t}$, where $t$ is the last index of the particular  inner loop, instead of randomly selected intermediate index. 
 \begin{algorithm} 
   \caption{SARAH+}
   \label{sarah_sc}
\begin{algorithmic}
   \STATE {\bfseries Parameters:} the learning rate $\eta > 0$, $0 < \gamma \leq 1$ and the maximum inner loop size $m$. 
   \STATE {\bfseries Initialize:} $\tilde{w}_0$
   \STATE {\bfseries Iterate:}
   \FOR{$s=1,2,\dots$}
   \STATE $w_0 = \tilde{w}_{s-1}$
   \STATE $v_0 = \frac{1}{n}\sum_{i=1}^{n} \nabla f_i(w_0)$
   \STATE $w_1 = w_0 - \eta v_0$
   \STATE $t = 1$
   \WHILE{$\|v_{t-1}\|^2 >  \gamma \|v_{0}\|^2$ {\bf and} $t<m$}
   \STATE Sample $i_{t}$ uniformly at random from $\setn$
   \STATE $v_{t} = \nabla f_{i_{t}} (w_{t}) - \nabla f_{i_{t}}(w_{t-1}) + v_{t-1}$
   \STATE $w_{t+1} = w_{t} - \eta v_{t}$
   \STATE $t = t + 1$
   \ENDWHILE
   \STATE Set $\tilde{w}_s = w_{t}$
   \ENDFOR
\end{algorithmic}
\end{algorithm} 
 
Different from SARAH, SARAH+ provides a possibility of earlier termination and unnecessary careful choices of $m$, and it also covers the classical gradient descent when we set $\gamma = 1$ (since the while loop does not proceed). In Figure~\ref{fig:SARAHplus} we  present the numerical performance of  SARAH+ with different $\gamma$s on \emph{rcv1} and \emph{news20} datasets. The size of the inner loop provides a trade-off between the fast sub-linear convergence in the inner loop and linear convergence in the outer loop. From the results, it appears that $\gamma=1/8$ is the optimal choice. With a larger $\gamma$, i.e. $\gamma > 1/8$, the iterates in the inner loop do not provide sufficient reduction, before another full gradient computation is required, while with  $\gamma < 1/8$
 an unnecessary number of inner steps is performed without gaining substantial progress. Clearly $\gamma$ is another parameter that requires tuning, however, in our experiments, the performance of SARAH+ has been very robust with respect to the choices of $\gamma$ and did not vary much from one data set to another. 
 

Similarly to SVRG, $\|v_t\|^2$ decreases in the outer iterations of SARAH+. However, unlike SVRG, SARAH+ also inherits from SARAH the consistent decrease of $\|v_t\|^2$ in expectation in the inner loops. It is not possible to apply the same idea of adaptively terminating the inner loop of SVRG based on the reduction in $\|v_t\|^2$, as $\|v_t\|^2$ may have side fluctuations as shown in Figure~\ref{fig:VR2}.

  \begin{figure} 
 \epsfig{file=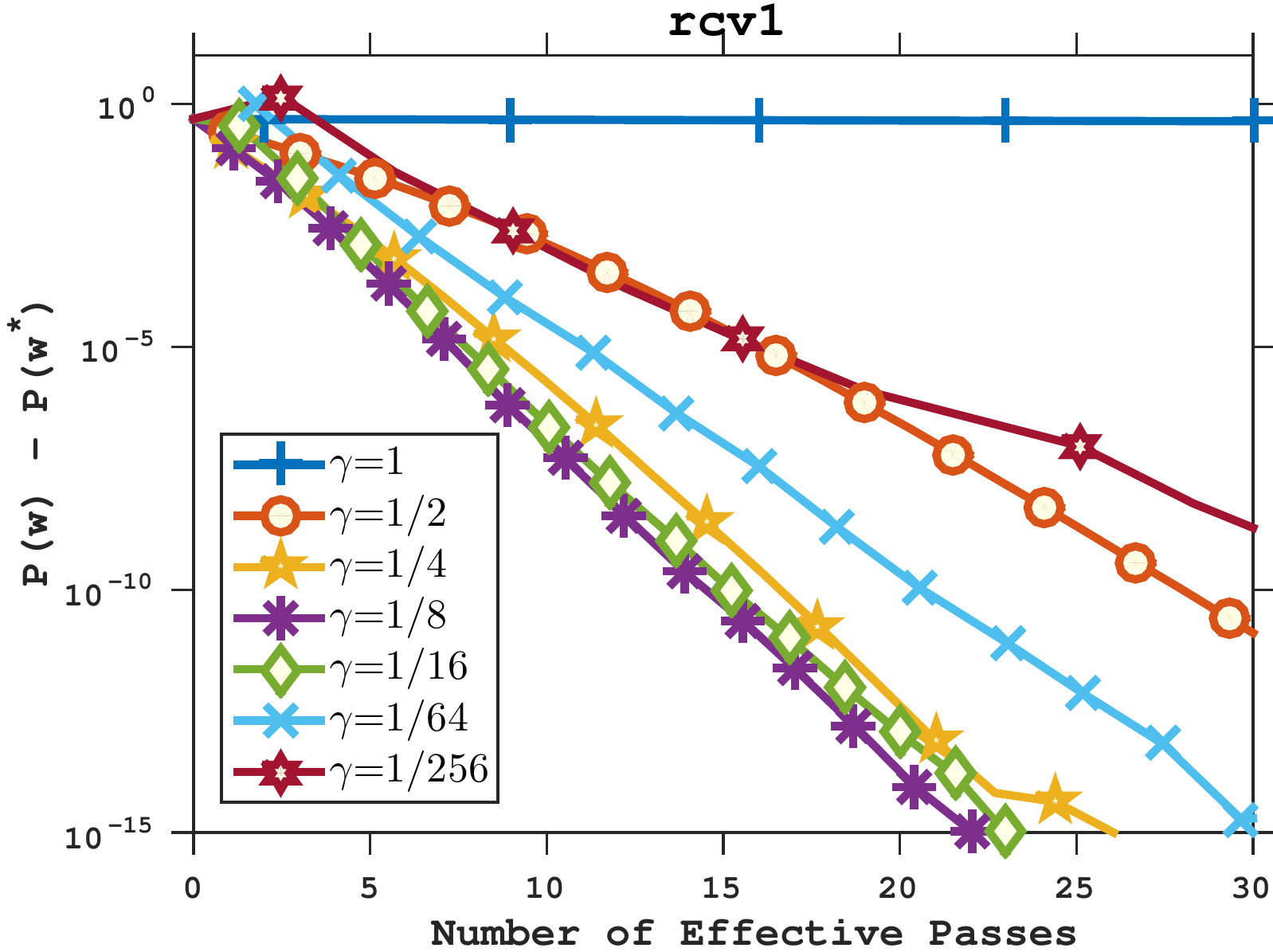,width=0.23\textwidth}
 \epsfig{file=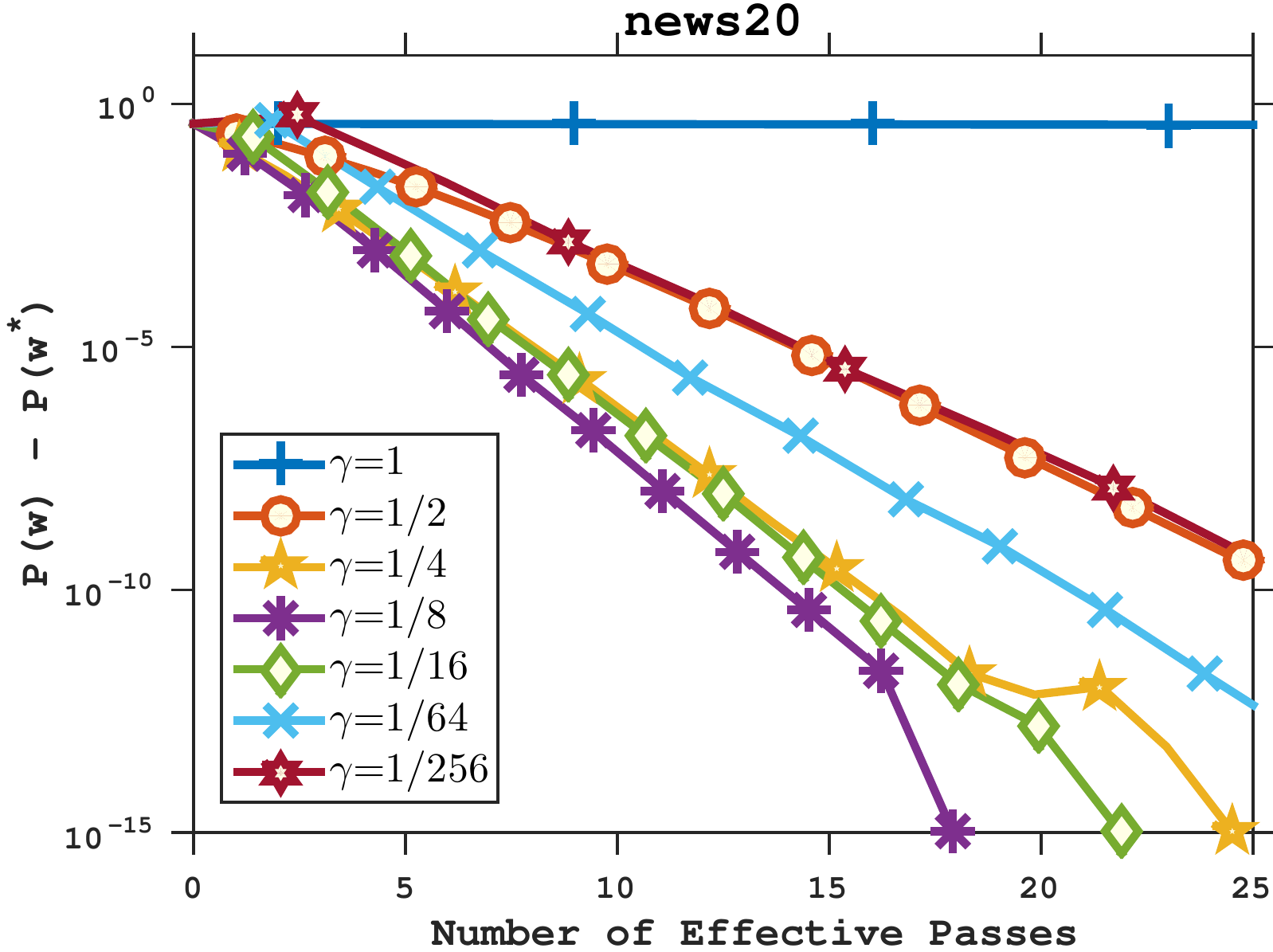,width=0.23\textwidth}
    \caption{\footnotesize An example of $\ell_2$-regularized logistic regression on \emph{rcv1} (left) and \emph{news20} (right) training datasets for SARAH+ with different $\gamma$s on loss residuals $P(w)-P(w^*)$.}
  \label{fig:SARAHplus}
 \end{figure}
 
\section{Numerical Experiments} 
 \begin{figure*} 
\centering
 \epsfig{file=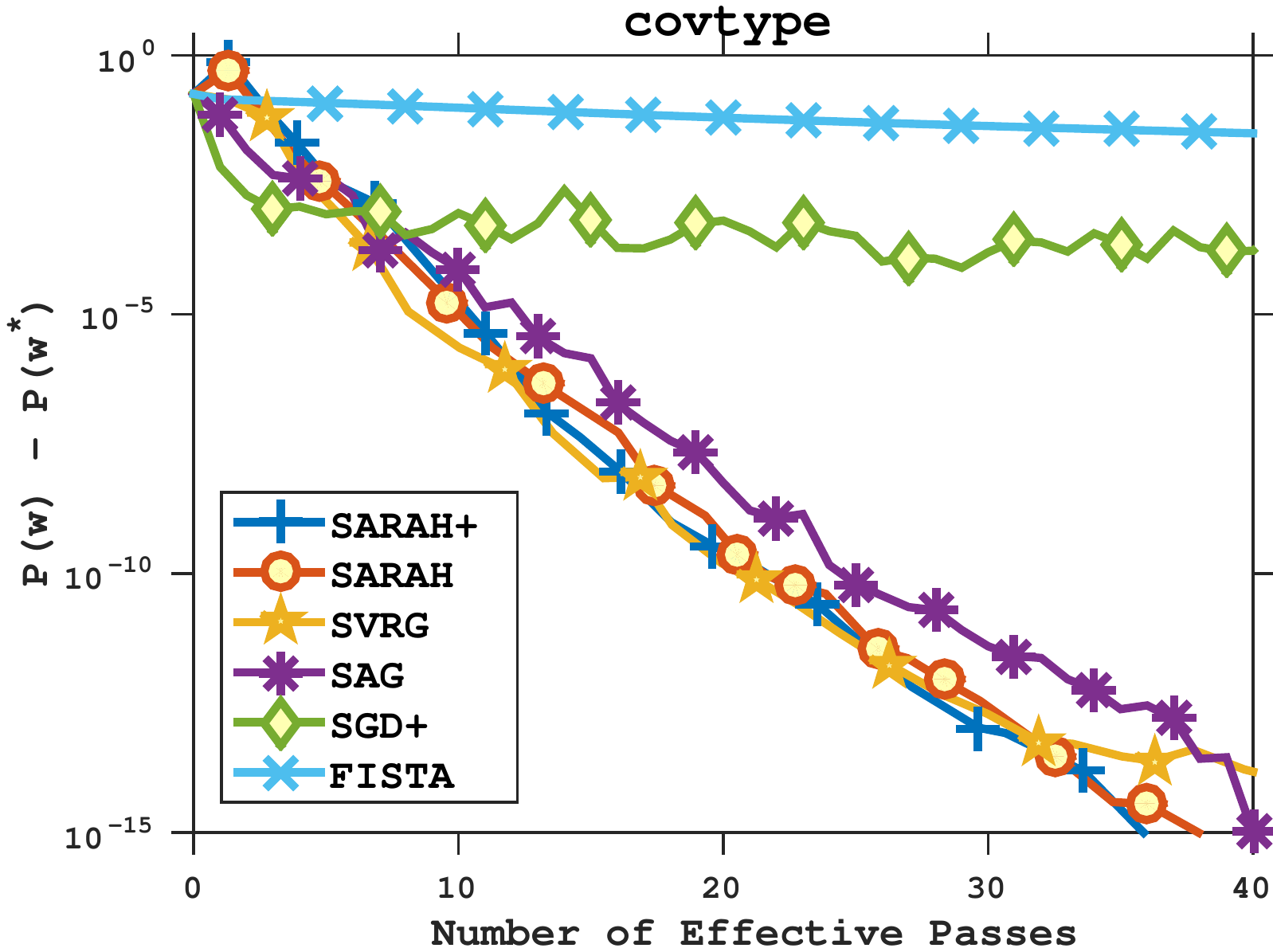,width=0.23\textwidth}
  \epsfig{file=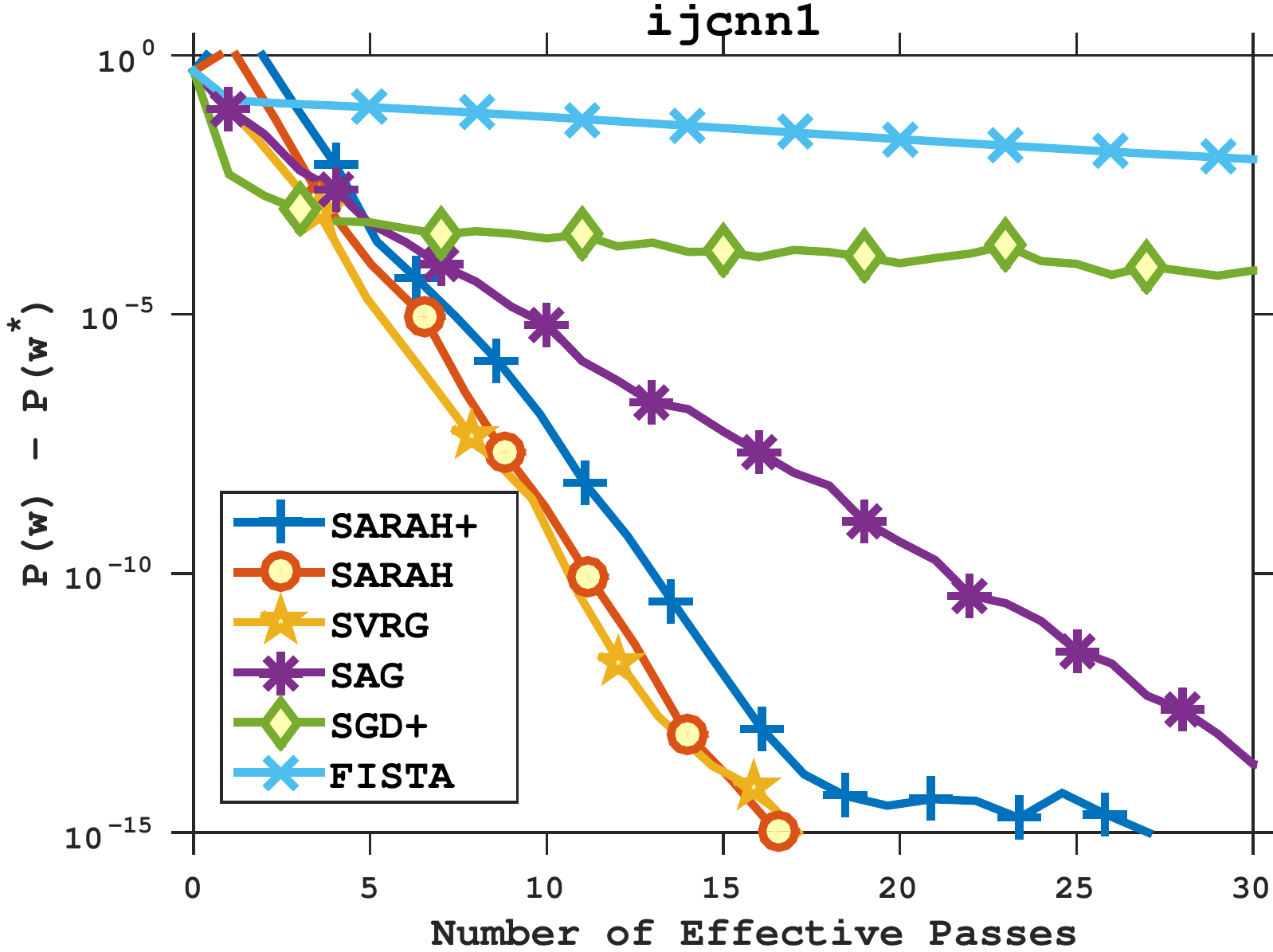,width=0.23\textwidth} 
  \epsfig{file=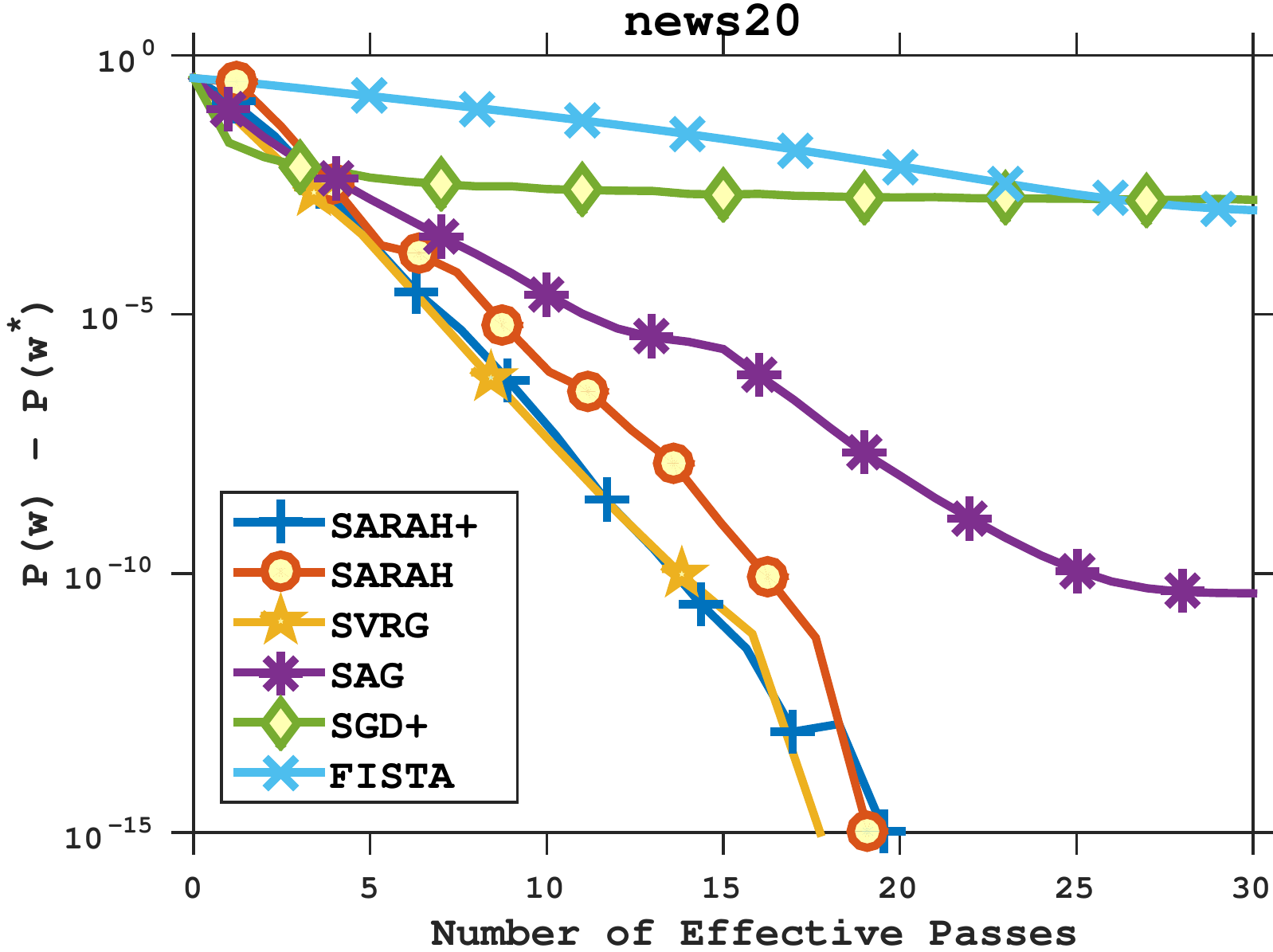,width=0.23\textwidth} 
  \epsfig{file=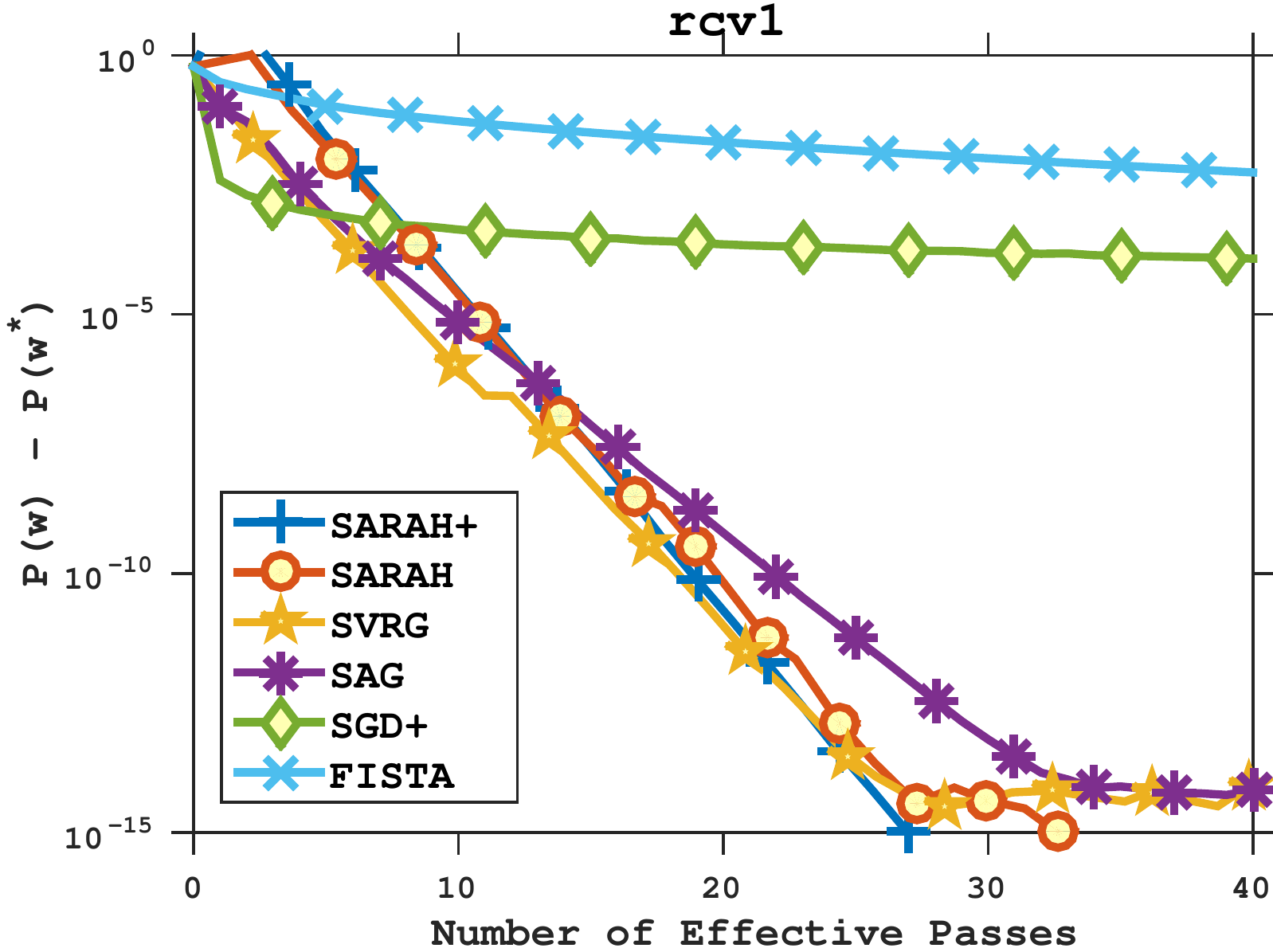,width=0.23\textwidth}

   \epsfig{file=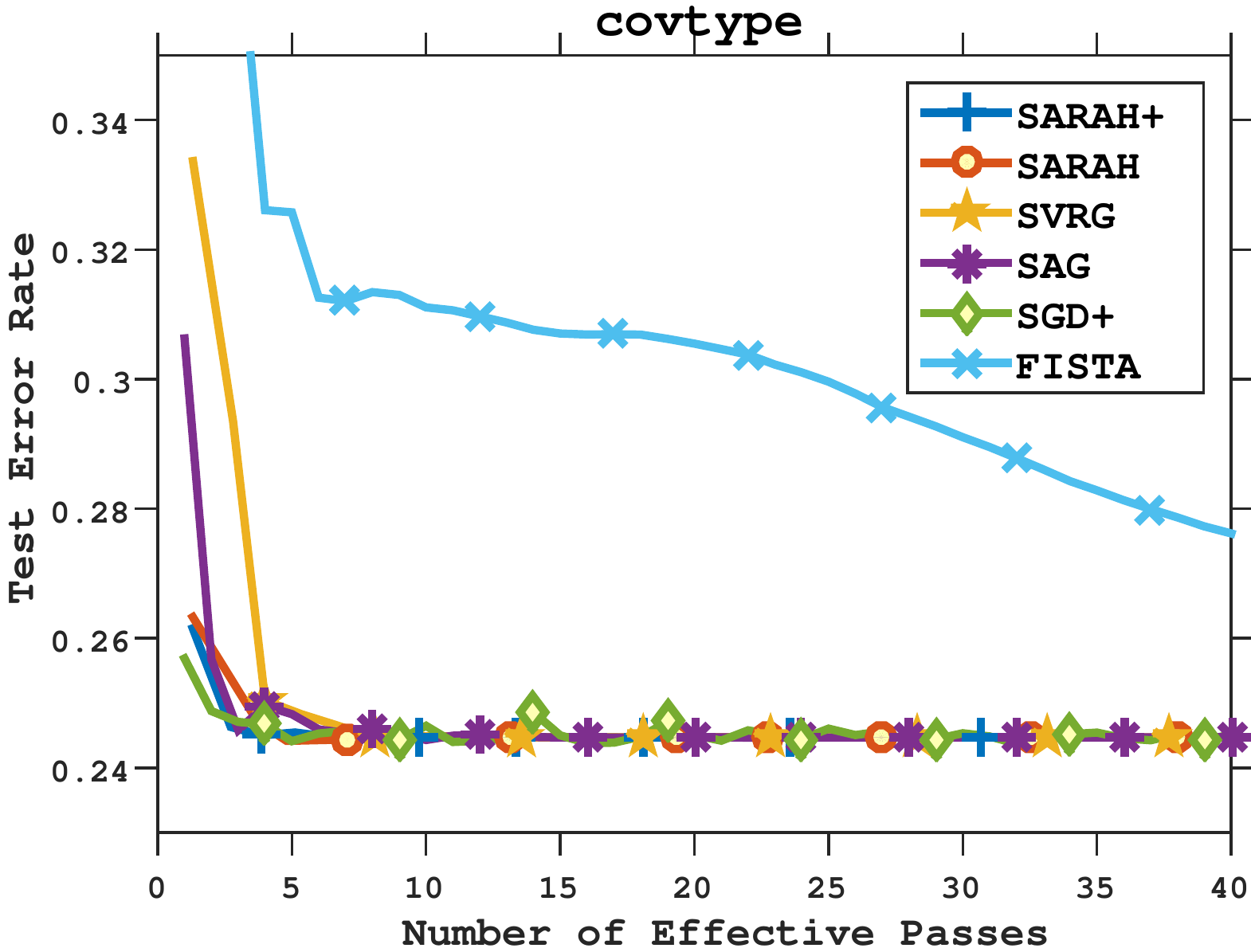,width=0.23\textwidth}
   \epsfig{file=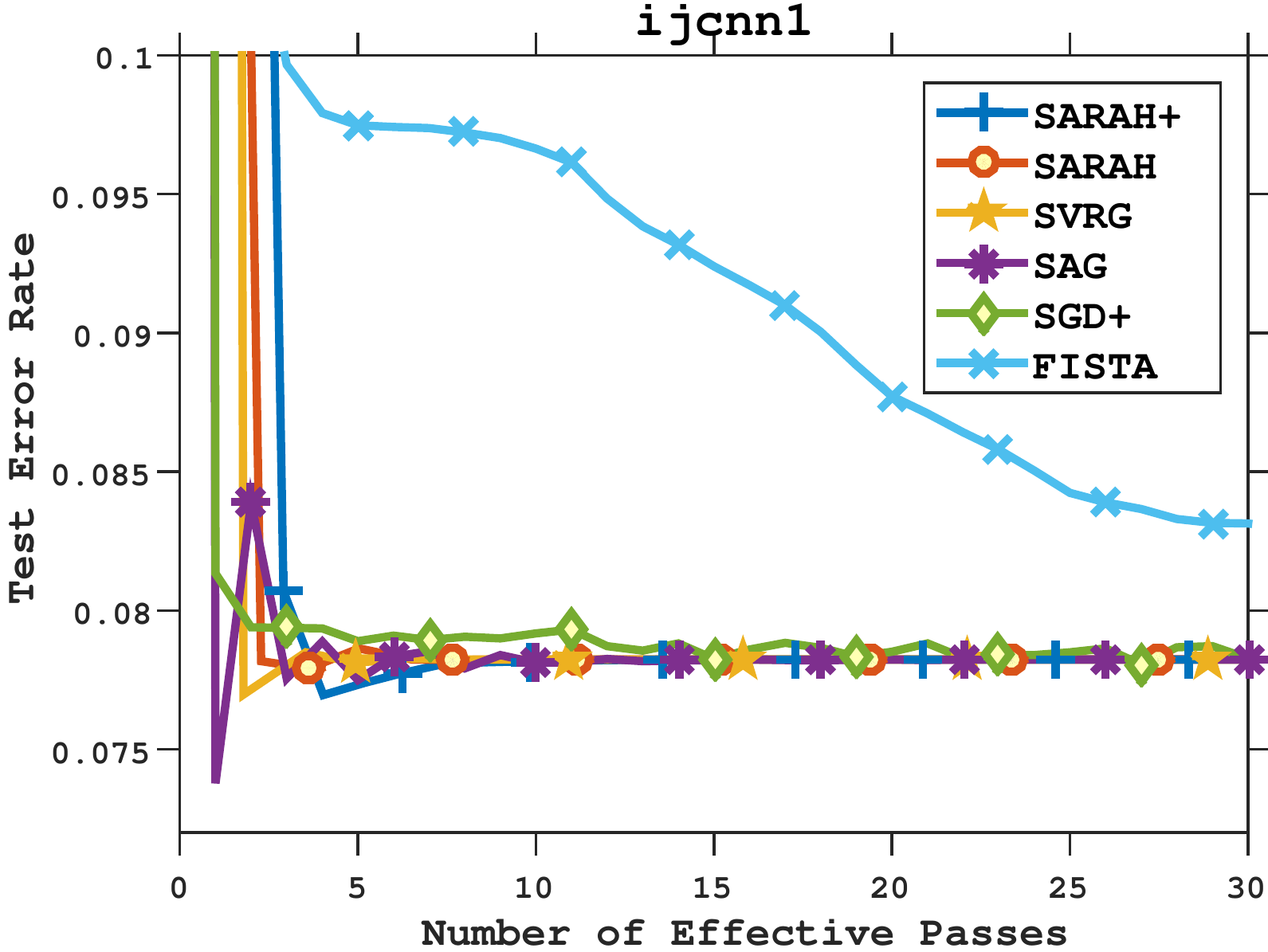,width=0.23\textwidth} 
    \epsfig{file=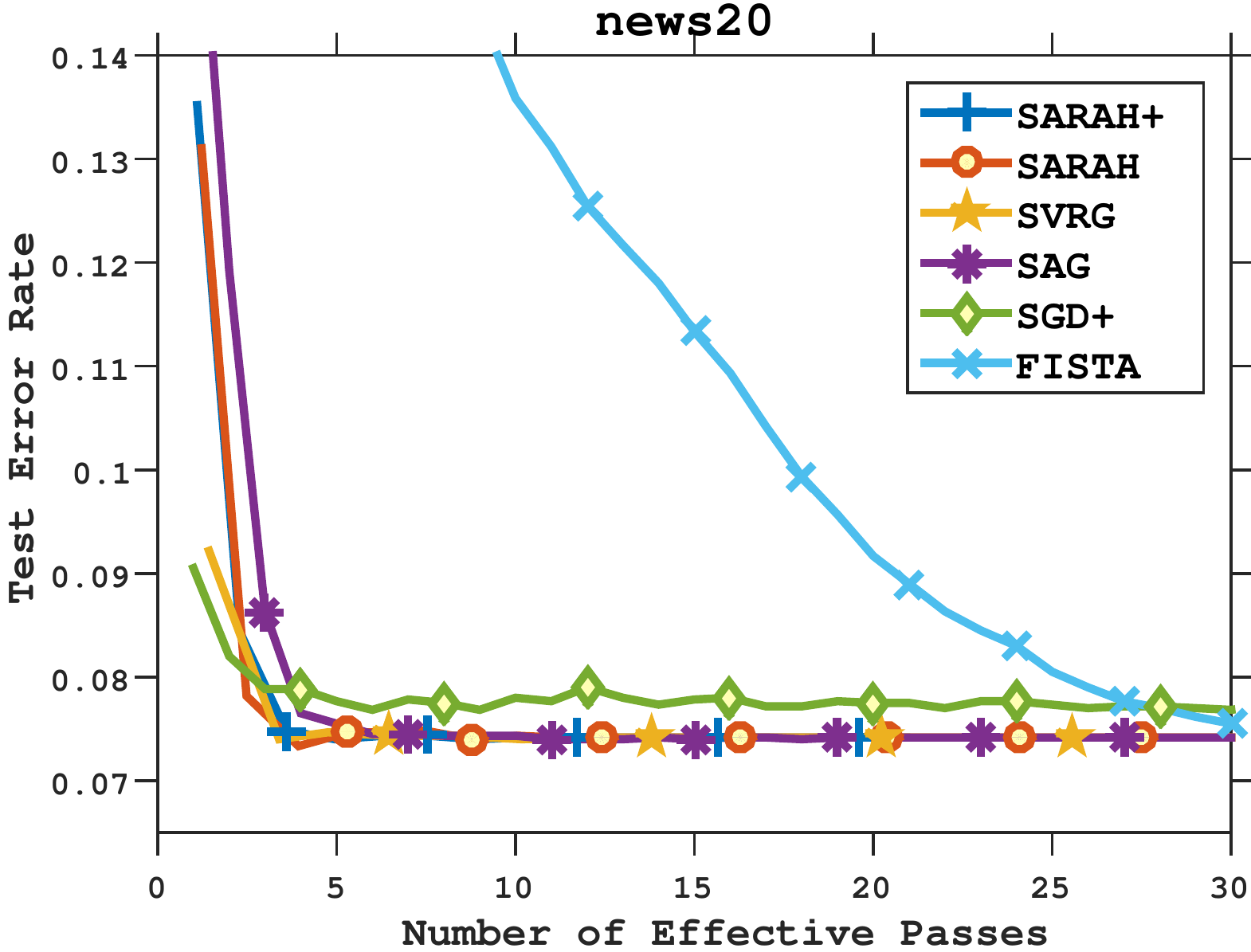,width=0.23\textwidth}
    \epsfig{file=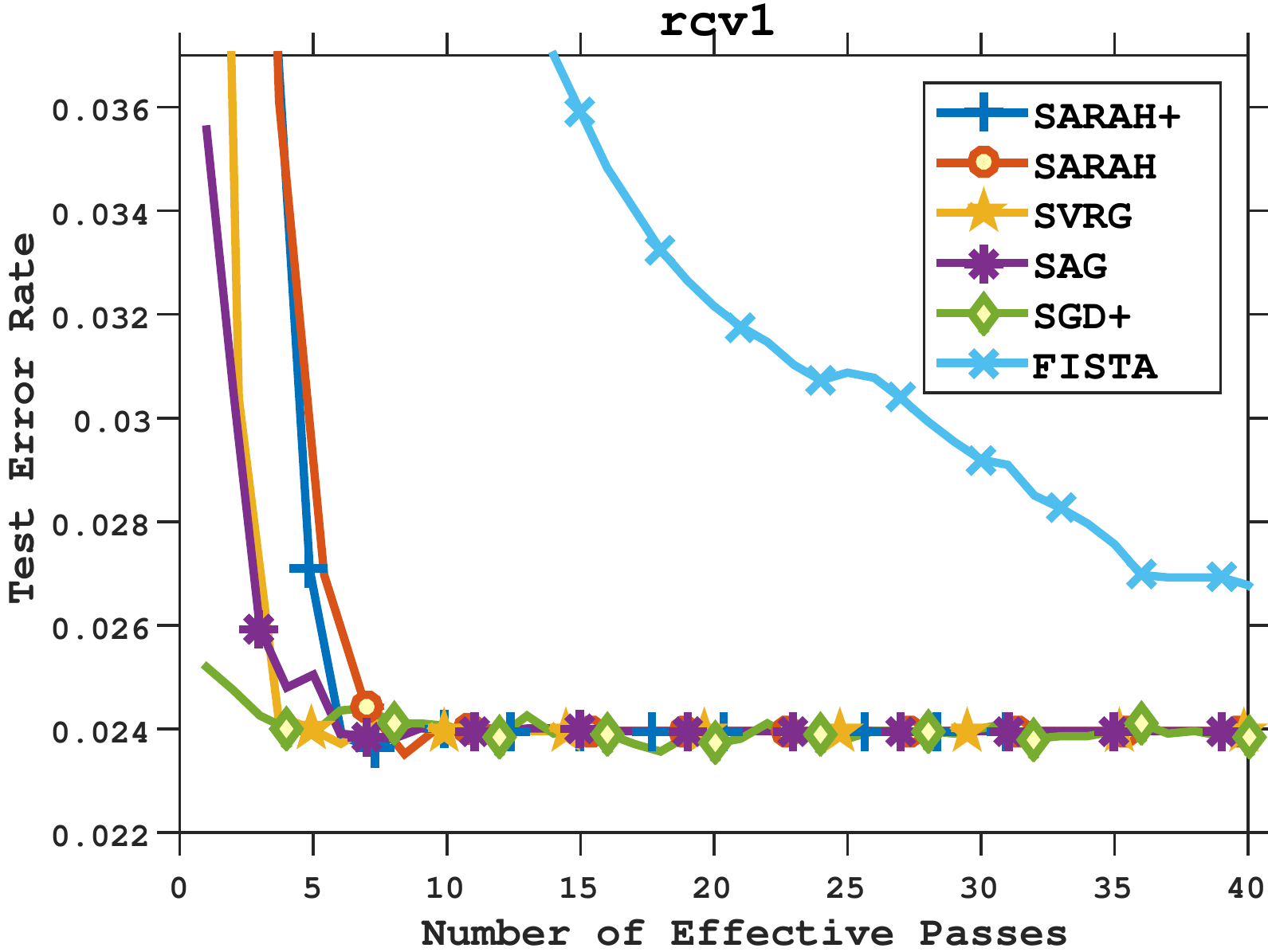,width=0.23\textwidth}
 \caption{\footnotesize Comparisons of loss residuals $P(w) - P(w^*)$ (top) and test errors (bottom) from different modern stochastic methods on \emph{covtype, ijcnn1, news20} and \emph{rcv1}.}
   \label{fig:test_errors}
 \end{figure*}

To support the theoretical analyses and insights, we present our empirical experiments, comparing SARAH and SARAH+  with the state-of-the-art first-order methods for $\ell_2$-regularized logistic regression problems with 
$$f_i(w) = \log (1+\exp(-y_ix_i^Tw)) + \tfrac{\lambda}{2}\|w\|^2,$$
on datasets \emph{covtype, ijcnn1, news20} and \emph{rcv1}~\footnote{All datasets are available at \burl{http://www.csie.ntu.edu.tw/~cjlin/libsvmtools/datasets/}.}. For \emph{ijcnn1} and \emph{rcv1}  we use the predefined testing and training sets, while  \emph{covtype} and \emph{news20} do not have  test data, hence we randomly   split the  datasets with $70\%$ for training and $30\%$ for testing. Some statistics of the datasets are summarized in Table~\ref{table:datasets}.
\begin{table} 
\scriptsize
\centering
\caption{Summary of datasets used for experiments.}
\label{table:datasets}
\begin{tabular}{|c|c|c|c|c|c|c|}
\hline
Dataset  & $d$  & $n$ (train) & Sparsity  & $n$ (test)   & $L$ \\
\hline \hline 
\emph{covtype} & 54  & 406,709  & 22.12\% & 174,303 & 1.90396 \\
\hline 
\emph{ijcnn1} & 22  & 91, 701  & 59.09\% & 49, 990 & 1.77662 \\
\hline
\emph{news20} & 1,355,191  & 13, 997  & 0.03375\% & 5, 999 & 0.2500 \\
\hline
\emph{rcv1} & 47,236 & 677,399 & 0.1549\%  & 20,242 & 0.2500\\
\hline 
\end{tabular}
\end{table}

 The penalty parameter $\lambda$ is set to $1/n$ as is common practice \cite{SAG}. Note that like SVRG/S2GD and SAG/SAGA, SARAH also allows an efficient sparse implementation named ``lazy updates"~\cite{konecny2015mini}.
We conduct and compare numerical results of SARAH with SVRG, SAG, SGD+ and FISTA. SVRG~\cite{SVRG} and SAG~\cite{SAG} are classic modern stochastic methods. SGD+ is SGD with decreasing learning rate $\eta=\eta_0/(k+1)$ where $k$ is the number of effective passes and $\eta_0$ is some initial constant learning rate. FISTA~\cite{fista} is the Fast Iterative Shrinkage-Thresholding Algorithm, well-known as an efficient accelerated version of the gradient descent. Even though for each method, there is a theoretical safe learning rate, we compare the results for the best learning rates in hindsight. 


Figure~\ref{fig:test_errors} shows numerical results in terms of loss residuals (top) and test errors (bottom) on the four datasets, SARAH is sometimes comparable or a little worse than other methods at the beginning. However, it quickly catches up to or surpasses all other methods, demonstrating a faster rate of decrease across all experiments. We observe that  on \emph{covtype} and \emph{rcv1}, SARAH, SVRG and SAG are comparable with some advantage of SARAH on  \emph{covtype}. On \emph{ijcnn1} and \emph{news20}, SARAH and SVRG consistently surpass the other methods. 

\begin{table} 
\scriptsize
\centering
\caption{Summary of best parameters for all the algorithms on different datasets.}
\label{table:stats}
\begin{tabular}{|c|C{1.25cm}|C{1.25cm}|C{0.7cm}|C{0.7cm}|C{0.7cm}|}
\hline
Dataset  & SARAH $(m^*,\eta^*)$ & SVRG $(m^*,\eta^*)$ & SAG ($\eta^*$)  & SGD+ ($\eta^*$)   & FISTA ($\eta^*$) \\
\hline \hline 
\emph{covtype} & (2n, 0.9/L)&  (n, 0.8/L) & 0.3/L & 0.06/L & 50/L\\ 
\hline 
\emph{ijcnn1} & (0.5n, 0.8/L)  & (n, 0.5/L) & 0.7/L & 0.1/L & 90/L \\
\hline
\emph{news20} & (0.5n, 0.9/L)  & (n, 0.5/L) & 0.1/L & 0.2/L & 30/L \\
\hline
\emph{rcv1} & (0.7n, 0.7/L) & (0.5n, 0.9/L) & 0.1/L & 0.1/L & 120/L \\ 
\hline 
\end{tabular}
\end{table}

In particular, to validate the efficiency of our practical variant SARAH+, we provide an insight into how important the choices of $m$ and $\eta$ are for SVRG and SARAH in Table~\ref{table:stats} and Figure~\ref{fig:ms}. Table~\ref{table:stats} presents the optimal choices of $m$ and $\eta$ for each of the algorithm, while Figure~\ref{fig:ms} shows the behaviors of SVRG and SARAH with different choices of $m$ for \emph{covtype} and \emph{ijcnn1}, where $m^*$s denote the best choices. 
In Table~\ref{table:stats}, the optimal learning rates of SARAH vary less among different datasets compared to all the other methods and they approximate the theoretical  upper bound for SARAH ($1/L$); on the contrary, for the other methods the empirical optimal rates can exceed their theoretical limits (SVRG with $1/(4L)$, SAG with $1/(16L)$, FISTA with $1/L$). This empirical studies suggest that it is much easier to tune and find the ideal learning rate for SARAH. As observed in Figure~\ref{fig:ms}, the behaviors of both SARAH and SVRG are quite sensitive to the choices of $m$. With improper choices of $m$, the loss residuals can be increased considerably from $10^{-15}$ to $10^{-3}$ on both \emph{covtype} in 40 effective passes and \emph{ijcnn1} in 17 effective passes for SARAH/SVRG.


 \begin{figure} 
\centering
 \epsfig{file=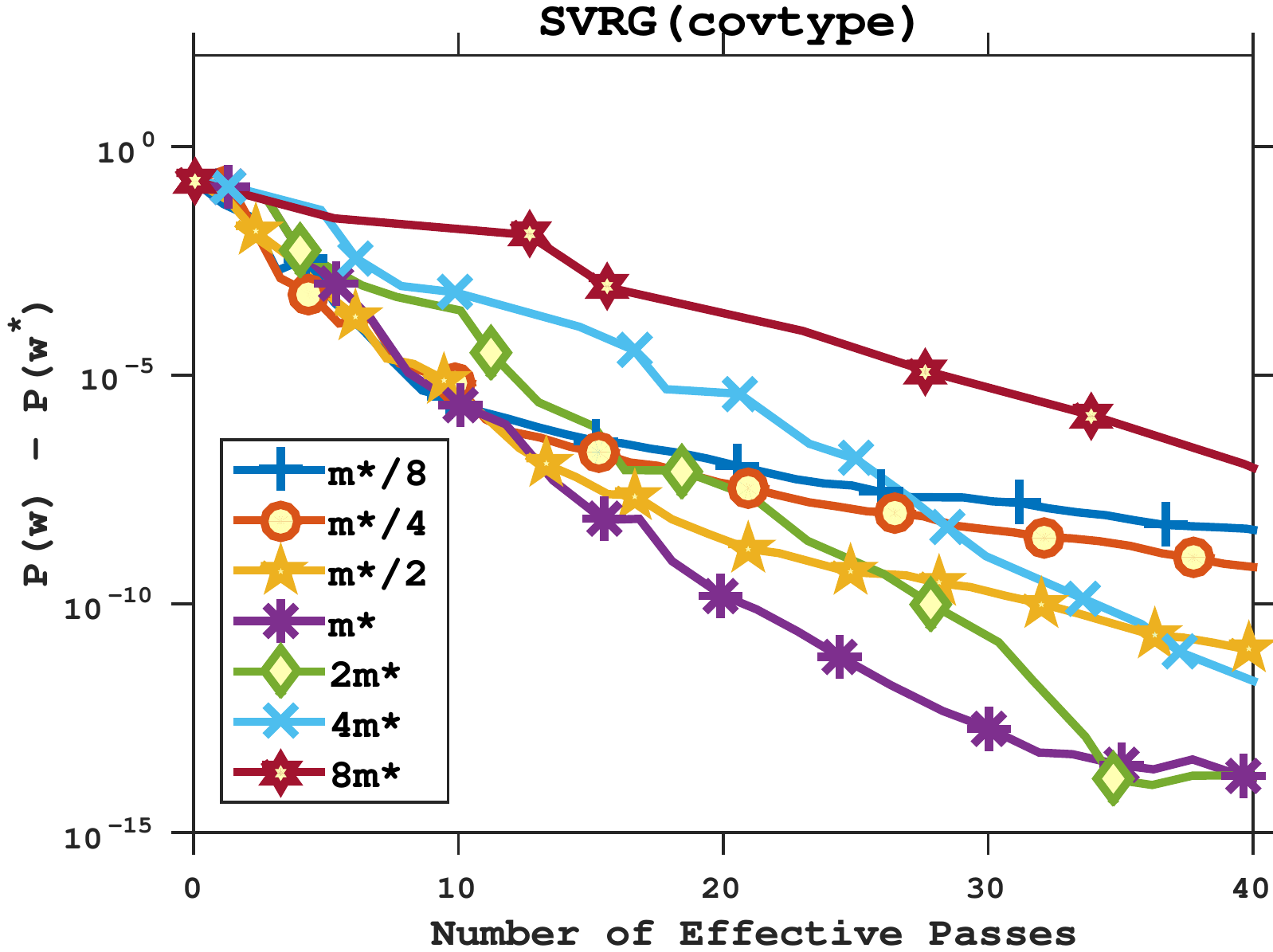,width=0.23\textwidth}
   \epsfig{file=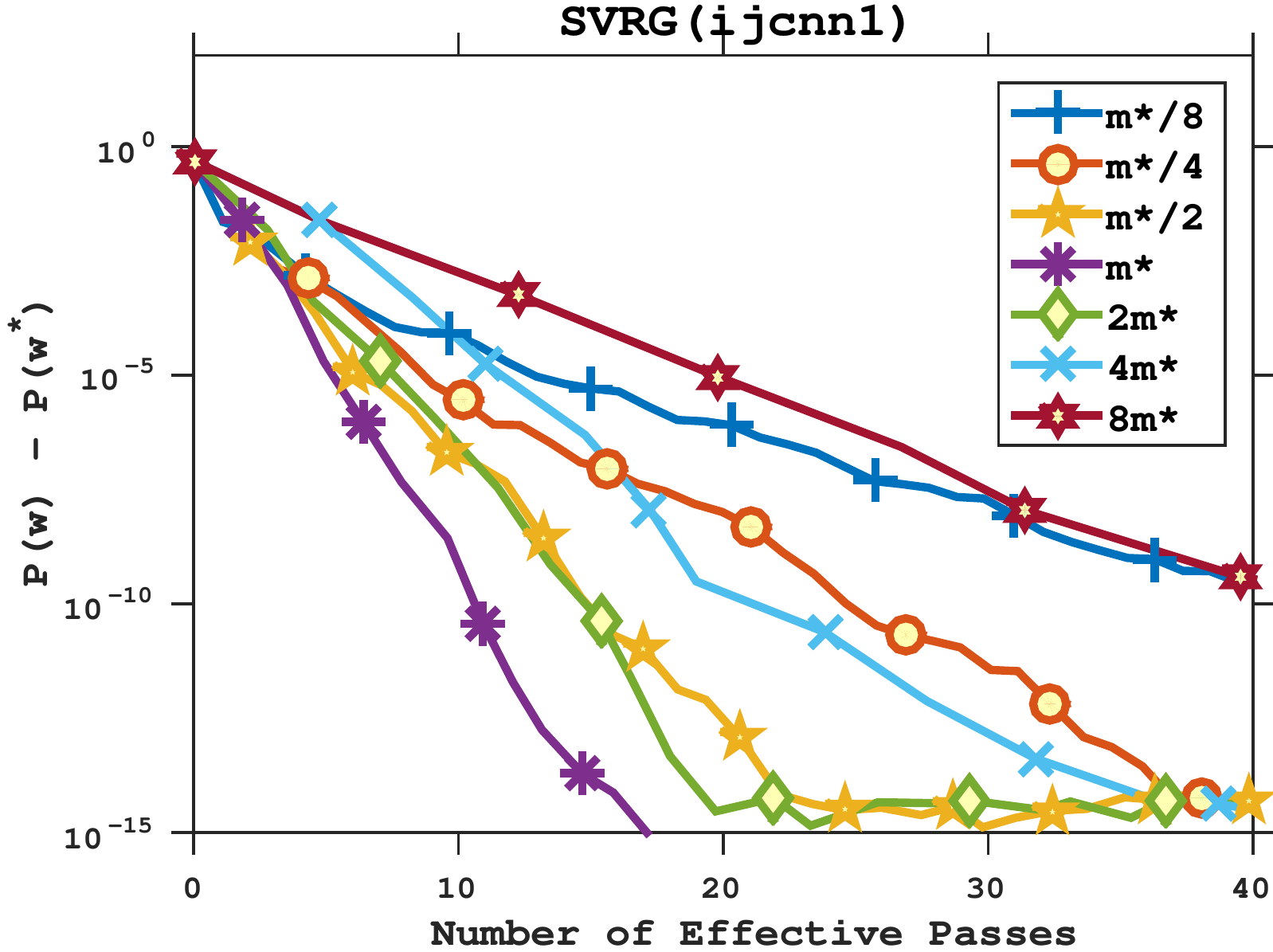,width=0.23\textwidth}
 \epsfig{file=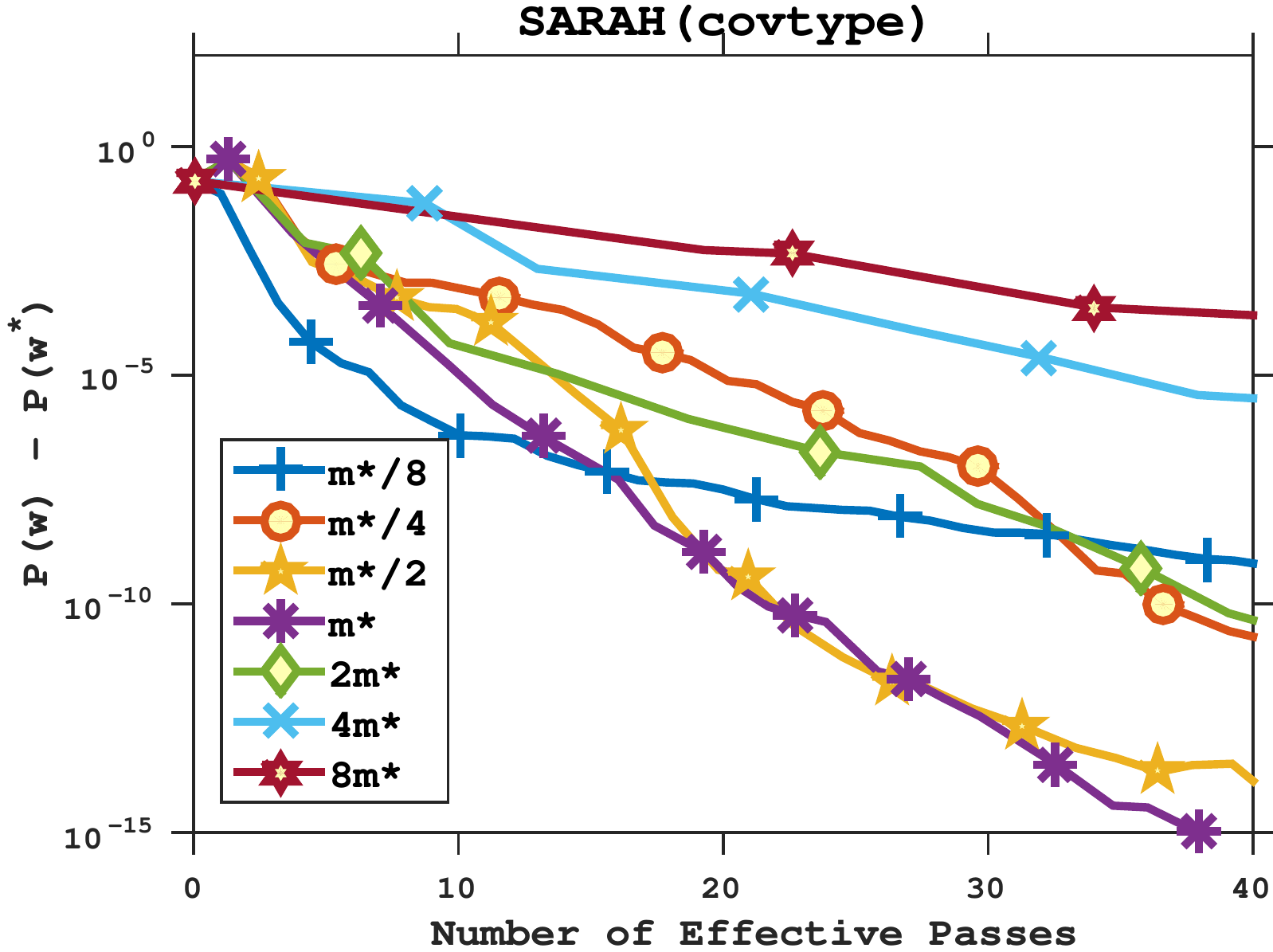,width=0.23\textwidth}
  \epsfig{file=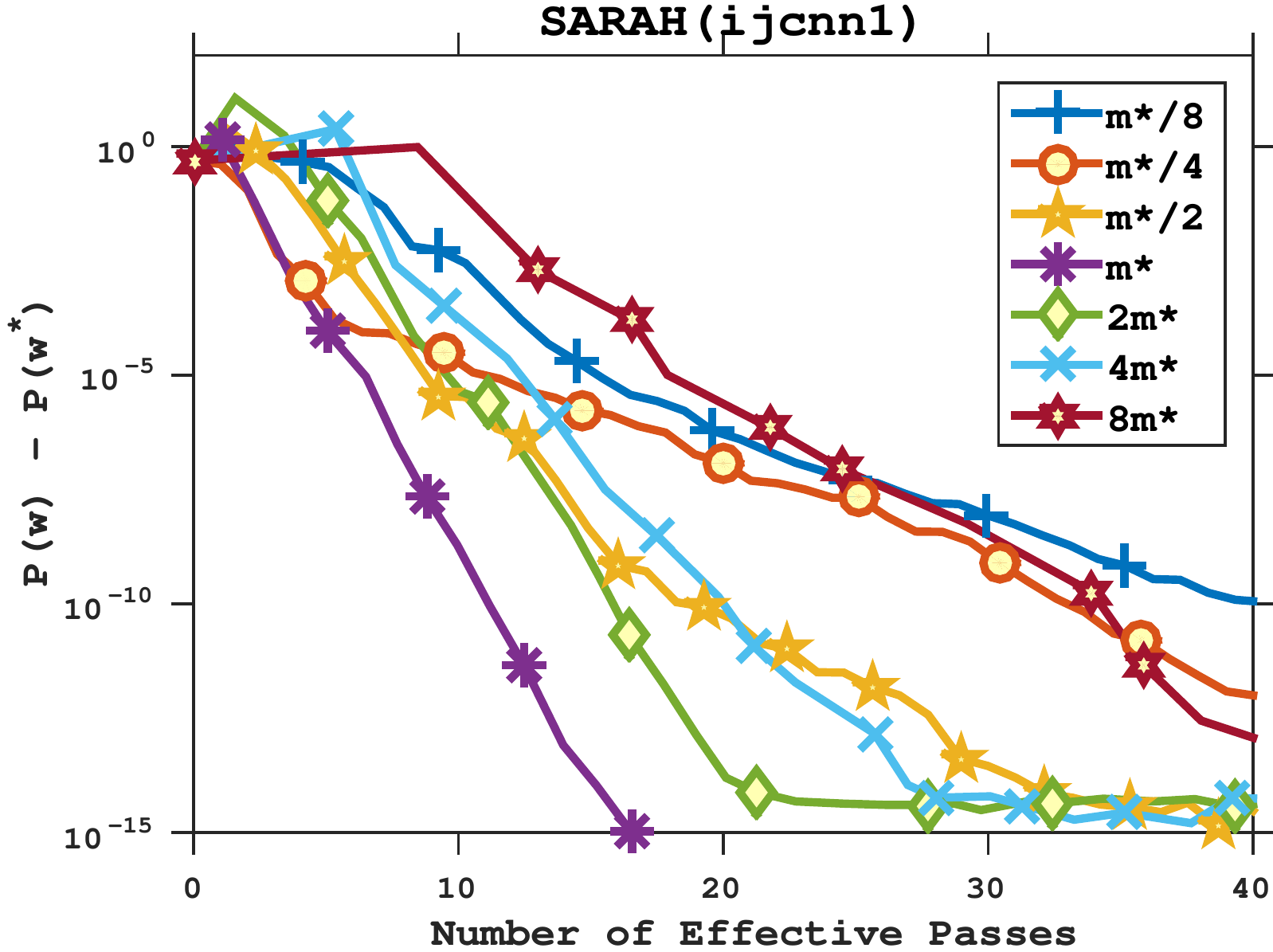,width=0.23\textwidth}
 \caption{\footnotesize Comparisons of loss residuals $P(w) - P(w^*)$ for different inner loop sizes with SVRG (top) and SARAH (bottom) on \emph{covtype} and \emph{ijcnn1}.}
   \label{fig:ms}
 \end{figure}


\section{Conclusion}
We propose a new variance reducing stochastic recursive gradient algorithm SARAH, which combines some of the properties of well known existing algorithms, such as SAGA and SVRG.   For smooth convex functions, we show a sublinear convergence rate, while for strongly convex cases, we prove the linear convergence rate and the  computational complexity as those of SVRG and SAG. However, compared to SVRG, SARAH's convergence rate constant is smaller and the algorithms is  more stable both theoretically and numerically. Additionally, we prove the linear convergence for inner loops of SARAH which support the claim of stability. Based on this convergence we derive a practical version of SARAH, with a simple stopping criterion for the inner loops.


\section*{Acknowledgements} 
The authors would like to thank the reviewers for useful suggestions which helped to improve the exposition in the paper. 

%
%


\bibliography{icml2017}
\bibliographystyle{template/icml2017}


\clearpage
\onecolumn
\appendix

\icmltitle{Supplementary Material}

\section{Technical Results}

\begin{lem}[Theorem 2.1.5 in \cite{nesterov2004}]
\label{lem_tech_01}
Suppose that $f$ is convex and $L$-smooth. Then, for any $w$, $w' \in \mathbb{R}^d$, 
\begin{gather*}
f(w) \leq f(w') + \nabla f(w')^T(w-w') + \frac{L}{2}\|w-w'\|^2,
\tagthis\label{eq:Lsmooth}
\\
f(w) \geq f(w') + \nabla f(w')^\top (w - w') + \frac{1}{2L}\| \nabla f(w) - \nabla f(w') \|^2, 
\tagthis 
\\
(\nabla f(w) - \nabla f(w'))^\top (w - w') \geq \frac{1}{L}\| \nabla f(w) - \nabla f(w') \|^2. 
\tagthis\label{ineq_convex} 
\end{gather*}
\end{lem}

Note that \eqref{eq:Lsmooth} does not require the convexity of $f$. 

\begin{lem}[Theorem 2.1.11 in \cite{nesterov2004}]
\label{lem_convex_lowerbound}
Suppose that $f$ is $\mu$-strongly convex and $L$-smooth. Then, for any $w$, $w' \in \mathbb{R}^d$, 
\begin{align*}
&(\nabla f(w) - \nabla f(w'))^\top (w - w') \geq \frac{\mu L}{\mu + L} \|w - w'\|^2  
 + \frac{1}{\mu + L}\| \nabla f(w) - \nabla f(w') \|^2. \tagthis\label{eqdasfsadfsa} 
\end{align*}
\end{lem}

\begin{lem}[Choices of $m$ and $\eta$] \label{lem:complexity}
Consider the rate of convergence $\sigma_m$ in Theorem \ref{thm:stronglyconvexconvergence}. If we choose $\eta = 1/(\theta L)$ with $\theta>1$ and fix $\sigma_m$, then the best choice of $m$ is 
$$m^*=\frac{1}{2}(2\theta^*-1)^2\kappa - 1,$$
where $\kappa\eqdef L/\mu,$ with $\theta^*$ calculated as:
$$\theta^* = \frac{\sigma_m+1+\sqrt{\sigma_m+1}}{2\sigma_m}.$$
Furthermore, we require $\theta^* > 1+\sqrt{2}/2$ for $\sigma_m<1$. 
\end{lem}

\section{Proofs}

\subsection{Proof of Lemma \ref{lem_main_derivation}}

By Assumption \ref{ass_Lsmooth} and $w_{t+1} = w_{t} - \eta v_{t}$, we have
\begin{align*}
\mathbb{E}[ P(w_{t+1})] & \overset{\eqref{eq:Lsmooth}}{\leq}  \mathbb{E}[ P(w_{t})] - \eta \mathbb{E}[\nabla P(w_{t})^\top v_{t}] 
+ \frac{L\eta^2}{2} \mathbb{E} [ \| v_{t} \|^2 ] 
\\
& = \mathbb{E}[ P(w_{t})] - \frac{\eta}{2} \mathbb{E}[ \| \nabla P(w_{t})\|^2 ] 
+ \frac{\eta}{2} \mathbb{E}[ \| \nabla P(w_{t}) - v_{t} \|^2 ] 
- \left( \frac{\eta}{2} - \frac{L\eta^2}{2} \right) \mathbb{E} [ \| v_{t} \|^2 ],
\end{align*}
where the last equality follows from the fact
$a^Tb = \frac{1}{2}\left[\|a\|^2 + \|b\|^2 - \|a-b\|^2\right].$

By summing over $t = 0,\dots,m$, we have
\begin{align*}
\mathbb{E}[ P(w_{m + 1})] & \leq  \mathbb{E}[ P(w_{0})] - \frac{\eta}{2} \sum_{t=0}^{m} \mathbb{E}[ \| \nabla P(w_{t})\|^2 ] + \frac{\eta}{2} \sum_{t=0}^{m} \mathbb{E}[ \| \nabla P(w_{t}) - v_{t} \|^2 ]  
- \left( \frac{\eta}{2} - \frac{L\eta^2}{2} \right) \sum_{t=0}^{m} \mathbb{E} [ \| v_{t} \|^2 ],  
\end{align*}
which is equivalent to ($\eta>0$):
\begin{align*}
\sum_{t=0}^{m} \mathbb{E}[ \| \nabla P(w_{t})\|^2 ]  & \leq \frac{2}{\eta} \mathbb{E}[ P(w_{0}) - P(w_{m+1})] + \sum_{t=0}^{m} \mathbb{E}[ \| \nabla P(w_{t}) - v_{t} \|^2 ]  
 - ( 1 - L\eta ) \sum_{t=0}^{m} \mathbb{E} [ \| v_{t} \|^2 ] \\
& \leq \frac{2}{\eta} \mathbb{E}[ P(w_{0}) - P(w^{*})] + \sum_{t=0}^{m} \mathbb{E}[ \| \nabla P(w_{t}) - v_{t} \|^2 ]  
 - ( 1 - L\eta ) \sum_{t=0}^{m} \mathbb{E} [ \| v_{t} \|^2 ],    
\end{align*}

where the last inequality follows since $w^{*}$ is a global minimizer of $P$. 

\subsection{Proof of Lemma \ref{lem:var_diff_01}}
Note that $\mathcal{F}_{j}$ contains all the information of $w_{0},\dots,w_{j}$ as well as $v_0,\dots,v_{j-1}$. For $j \geq 1$, we have
\begin{align*}
\mathbb{E}[ \| \nabla P(w_{j}) - v_{j} \|^2 | \mathcal{F}_{j} ] 
& = \mathbb{E}[ \| [\nabla P(w_{j-1}) - v_{j-1} ] + [ \nabla P(w_{j}) - \nabla P(w_{j-1}) ]  - [ v_{j} - v_{j-1} ] \|^2 | \mathcal{F}_{j} ]
\\
& = \| \nabla P(w_{j-1}) - v_{j-1} \|^2 + \| \nabla P(w_{j}) - \nabla P(w_{j-1}) \|^2 + \mathbb{E} [ \| v_{j} - v_{j-1}  \|^2 | \mathcal{F}_{j} ] 
\\
&\quad + 2 ( \nabla P(w_{j-1}) - v_{j-1} )^\top ( \nabla P(w_{j}) - \nabla P(w_{j-1}) ) \\
&\quad - 2 ( \nabla P(w_{j-1}) - v_{j-1} )^\top \mathbb{E}[ v_{j} - v_{j-1} | \mathcal{F}_{j} ]  \\
&\quad - 2 ( \nabla P(w_{j}) - \nabla P(w_{j-1}) )^\top \mathbb{E}[ v_{j} - v_{j-1} | \mathcal{F}_{j} ] 
\\
& = \| \nabla P(w_{j-1}) - v_{j-1} \|^2 - \| \nabla P(w_{j}) - \nabla P(w_{j-1}) \|^2 + \mathbb{E} [ \| v_{j} - v_{j-1}  \|^2 | \mathcal{F}_{j} ], 
\end{align*}
where the last equality follows from
\begin{align*}
& \mathbb{E}[ v_{j} - v_{j-1} | \mathcal{F}_{j} ] \overset{\eqref{eq:vt}}{= }\mathbb{E}[ \nabla f_{i_{j}}(w_{j}) - \nabla f_{i_{j}}(w_{j-1}) | \mathcal{F}_{j} ] 
= \nabla P(w_{j}) - \nabla P(w_{j-1}).
\end{align*}

By taking expectation for the above equation, we have
\begin{align*}
\mathbb{E}[ \| \nabla P(w_{j}) - v_{j} \|^2 ] &= \mathbb{E}[ \| \nabla P(w_{j-1}) - v_{j-1} \|^2 ] - \mathbb{E}[ \| \nabla P(w_{j}) - \nabla P(w_{j-1}) \|^2 ] + \mathbb{E}[ \| v_{j} - v_{j-1} \|^2 ]. 
\end{align*}

Note that $\| \nabla P(w_{0}) - v_{0} \|^2 = 0$. By summing over $j = 1,\dots,t\ (t\geq 1)$, we have
\begin{align*}
& \mathbb{E}[ \| \nabla P(w_{t}) - v_{t} \|^2 ]  = \sum_{j = 1}^{t} \mathbb{E}[ \| v_{j} - v_{j-1} \|^2 ]  - \sum_{j = 1}^{t} \mathbb{E}[ \| \nabla P(w_{j}) - \nabla P(w_{j-1}) \|^2 ]. 
\end{align*}

\subsection{Proof of Lemma \ref{lem_bound_var_diff_str_02}}

For $j \geq 1$, we have
\begin{align*}
 \mathbb{E}[\| v_{j} \|^2 | \mathcal{F}_{j}] 
&= \mathbb{E}[\| v_{j-1} - (\nabla f_{i_{j}}(w_{j-1}) - \nabla f_{i_{j}}(w_{j}) ) \|^2 | \mathcal{F}_{j} ] 
\\ 
&= \|v_{j-1} \|^2 + \mathbb{E}\Big[\| \nabla f_{i_{j}}(w_{j-1}) - \nabla f_{i_{j}}(w_{j}) \|^2  
 - \tfrac{2}{\eta}(\nabla f_{i_{j}}(w_{j-1}) - \nabla f_{i_{j}}(w_{j}))^\top (w_{j-1} - w_{j}) | \mathcal{F}_{j} \Big] 
 \\ 
& \overset{\eqref{ineq_convex}}{\leq} \|v_{j-1} \|^2 + \mathbb{E}\Big[\| \nabla f_{i_{j}}(w_{j-1}) - \nabla f_{i_{j}}(w_{j}) \|^2 
- \tfrac{2}{L \eta} \| \nabla f_{i_{j}}(w_{j-1}) - \nabla f_{i_{j}}(w_{j}) \|^2 | \mathcal{F}_{j} \Big] 
\\ 
& = \|v_{j-1} \|^2 + \left(1 - \tfrac{2}{\eta L}\right) \mathbb{E} [ \| \nabla f_{i_{j}}(w_{j-1}) - \nabla f_{i_{j}}(w_{j}) \|^2 | \mathcal{F}_{j} ] \\
& \overset{\eqref{eq:vt}}{= } \|v_{j-1} \|^2 + \left(1 - \tfrac{2}{\eta L}\right) \mathbb{E} [ \| v_{j} - v_{j-1} \|^2 | \mathcal{F}_{j} ],
\end{align*}
which, if we take expectation, implies that
$$
\mathbb{E}[\| v_{j} - v_{j-1} \|^2] 
\leq \frac{\eta L}{2 - \eta L} \Big[ \mathbb{E}[ \|v_{j-1} \|^2] - \mathbb{E}[\| v_{j} \|^2] \Big], 
$$
when $\eta < 2/{L}$.

\quad By summing the above inequality over $j = 1,\dots, t\ (t\geq 1)$, we have
\begin{equation}\label{eq:sumover}
\sum_{j=1}^{t} \mathbb{E}[\| v_{j} - v_{j-1} \|^2] 
\leq \frac{\eta L}{2 - \eta L} \Big[ \mathbb{E}[ \|v_{0} \|^2] - \mathbb{E}[\| v_{t} \|^2] \Big].  
\end{equation}

By Lemma \ref{lem:var_diff_01}, we have
\begin{align*}
 \mathbb{E}[ \| \nabla P(w_{t}) - v_{t} \|^2 ] & \leq \sum_{j = 1}^{t} \mathbb{E}[ \| v_{j} - v_{j-1} \|^2 ]  \overset{\eqref{eq:sumover}}{\leq}  \frac{\eta L}{2 - \eta L} \Big[ \mathbb{E}[ \|v_{0} \|^2] - \mathbb{E}[\| v_{t} \|^2] \Big]. 
\end{align*}

\subsection{Proof of Lemma~\ref{lem:complexity}}

With $\eta=1/(\theta L)$ and $\kappa=L/\mu$, the rate of convergence $\alpha_m$ can be written as
\begin{align*}
\sigma_m &\overset{\eqref{eq:sigma0}}{=} \frac{1}{\mu \eta (m + 1)} +  \frac{\eta L}{2 - \eta L}
= \frac{\theta L}{\mu(m+1)} + \frac{1/\theta}{2-1/\theta} 
   = \left(\frac{\kappa}{m+1}\right)\theta + \frac{1}{2\theta-1},
\end{align*}
which is equivalent to
$$m(\theta) \eqdef m = \frac{\theta(2\theta-1)}{\sigma_m (2\theta-1)-1} \kappa - 1.$$
Since $\sigma_m$ is considered fixed, then the optimal choice of $m$ in terms of $\theta$ can be solved from $\min_\theta m (\theta)$, or equivalently, $0 = (\partial m)/(\partial \theta) = m'(\theta)$, and therefore we have the equation with the optimal $\theta$ satisfying
\begin{equation}\label{Ksigma_m}
\sigma_m = (4\theta^*-1)/(2\theta^*-1)^2,
\end{equation}
and by plugging it into $m(\theta)$ we conclude the optimal $m$:
$$m^* = m(K^*) = \frac{1}{2}(2K^*-1)^2\kappa -1,$$
while by solving for $\theta^*$ in \eqref{Ksigma_m} and taking into account that $\theta>1$, we have the optimal choice of $\theta$:
$$\theta^* = \frac{\sigma_m+1+\sqrt{\sigma_m+1}}{2\sigma_m}.$$
Obviously, for $\sigma_m <1$, we require $\theta^* > 1+\sqrt{2}/2$.

\subsection{Proof of Theorem~\ref{lem_bouned_moment_stronglyconvexP}}

For $t \geq 1$, we have
\begin{align*}
\| \nabla P(w_{t}) - \nabla P(w_{t-1})\|^2 &= \Big\| \frac{1}{n} \sum_{i=1}^{n} [ \nabla f_i(w_{t}) - \nabla f_i(w_{t-1})  ] \Big\|^2 \\
& \leq \frac{1}{n} \sum_{i=1}^{n} \| \nabla f_i(w_{t}) - \nabla f_i(w_{t-1})  \|^2 \\
& = \mathbb{E}[ \| \nabla f_{i_{t}}(w_{t}) - \nabla f_{i_{t}}(w_{t-1}) \|^2 | \mathcal{F}_{t} ].  \tagthis \label{eq:prthm001}
\end{align*}

Using the proof of Lemma \ref{lem_bound_var_diff_str_02}, for $t \geq 1$, we have
\begin{align*}
 \mathbb{E}[\| v_{t} \|^2 | \mathcal{F}_{t}] 
& \leq \|v_{t-1} \|^2 + \left(1 - \tfrac{2}{\eta L}\right) \mathbb{E} [ \| \nabla f_{i_{t}}(w_{t-1}) - \nabla f_{i_{t}}(w_{t}) \|^2 | \mathcal{F}_{t} ] \\
& \overset{\eqref{eq:prthm001}}{\leq } \|v_{t-1} \|^2 + \left(1 - \tfrac{2}{\eta L}\right) \| \nabla P(w_{t}) - \nabla P(w_{t-1})\|^2 \\
& \leq \|v_{t-1} \|^2 + \left(1 - \tfrac{2}{\eta L}\right) \mu^2 \eta^2 \|v_{t-1}\|^2.
\end{align*}

Note that $1 - \tfrac{2}{\eta L} < 0$ since $\eta < 2/L$. The last inequality follows by the strong convexity of $P$, that is, $\mu \|w_{t} - w_{t-1}\| \leq \| \nabla P(w_{t}) - \nabla P(w_{t-1})\|$ and the fact that $w_{t} = w_{t-1} - \eta v_{t-1}$. By taking the expectation and applying recursively, we have
\begin{align*}
\mathbb{E}[\| v_{t} \|^2] & \leq \left[ 1 - \left(\tfrac{2}{\eta L} - 1 \right) \mu^2 \eta^2  \right] \mathbb{E}[\| v_{t-1} \|^2] \\
& \leq \left[ 1 - \left(\tfrac{2}{\eta L} - 1 \right) \mu^2 \eta^2  \right]^{t} \mathbb{E}[\| v_{0} \|^2] \\
& = \left[ 1 - \left(\tfrac{2}{\eta L} - 1 \right) \mu^2 \eta^2  \right]^{t} \mathbb{E}[\| \nabla P(w_{0}) \|^2]. 
\end{align*}

\subsection{Proof of Theorem~\ref{thm:bound_moment}}
We obviously have $\mathbb{E}[\|v_0\|^2 | \mathcal{F}_{0}] = \|\nabla P(w_0)\|^2$. For $t \geq 1$, we have
\begin{align*}
\mathbb{E}[\| v_{t} \|^2 | \mathcal{F}_{t}] 
& \overset{\eqref{eq:vt}}{=}  \mathbb{E}[\| v_{t-1} - (\nabla f_{i_{t}}(w_{t-1}) - \nabla f_{i_{t}}(w_{t}) ) \|^2 | \mathcal{F}_{t} ] \\
&\overset{\eqref{eq:iterate}}{=} \|v_{t-1} \|^2    + \mathbb{E}[\| \nabla f_{i_{t}}(w_{t-1}) - \nabla f_{i_{t}}(w_{t}) \|^2  - \tfrac{2}{\eta}(\nabla f_{i_{t}}(w_{t-1}) - \nabla f_{i_{t}}(w_{t}))^\top (w_{t-1} - w_{t}) | \mathcal{F}_{t} ] \\
&\overset{\eqref{eqdasfsadfsa}
}{ \leq}  \|v_{t-1} \|^2  
 - \tfrac{2 \mu L \eta}{\mu + L}  \|v_{t-1} \|^2  
 + \mathbb{E}[\| \nabla f_{i_{t}}(w_{t-1}) - \nabla f_{i_{t}}(w_{t}) \|^2 | \mathcal{F}_{t} ] - \tfrac{2}{\eta} \cdot \tfrac{1}{\mu + L} \mathbb{E}[\| \nabla f_{i_{t}}(w_{t-1}) - \nabla f_{i_{t}}(w_{t}) \|^2 | \mathcal{F}_{t} ] \\
&= 
(1 - \tfrac{2 \mu L \eta}{\mu + L})  \|v_{t-1} \|^2  
+ (1- \tfrac{2}{\eta} \cdot \tfrac{1}{\mu + L} ) \mathbb{E}[\| \nabla f_{i_{t}}(w_{t-1}) - \nabla f_{i_{t}}(w_{t}) \|^2 | \mathcal{F}_{t} ] \\
& \leq  \left( 1 - \tfrac{2 \mu L \eta}{\mu + L} \right)  \|v_{t-1} \|^2,  
\tagthis \label{eqasfewafaw}
\end{align*}
where in last inequality we have used that $\eta \leq 2/(\mu+L)$. By taking the expectation and applying recursively, the desired result is achieved. 

\subsection{Proof of Theorem~\ref{thm:generalconvex_02}}
By Theorem \ref{thm:generalconvex_01}, we have
\begin{align*}
\mathbb{E}[ \| \nabla P(\tilde{w}_s)\|^2 ] & \leq \frac{2}{\eta(m + 1)} \mathbb{E}[ P(\tilde{w}_{s-1}) - P(w^{*})]  + \frac{ \eta L}{2 - \eta L}   \mathbb{E}[ \| \nabla P(\tilde{w}_{s-1})\|^2 ] \\
& = \delta_{s-1} + \alpha \mathbb{E}[ \| \nabla P(\tilde{w}_{s-1})\|^2 ] \\
& \leq \delta_{s-1} + \alpha \delta_{s-2} + \dots + \alpha^{s-1} \delta_{0} + \alpha^s \| \nabla P(\tilde{w}_{0})\|^2 \\
& \leq \delta + \alpha \delta + \dots + \alpha^{s-1} \delta + \alpha^s \| \nabla P(\tilde{w}_{0})\|^2 \\
& \leq \delta \frac{1 - \alpha^s}{1 - \alpha} + \alpha^s \| \nabla P(\tilde{w}_{0})\|^2 \\
& = \Delta(1 - \alpha^s) + \alpha^s  \| \nabla P(\tilde{w}_0)\|^2 \\
& = \Delta + \alpha^s ( \| \nabla P(\tilde{w}_0)\|^2 - \Delta),
\end{align*}
where the second last equality follows since
\begin{align*}
\frac{\delta}{1 - \alpha} = \delta \left( \frac{2 - \eta L}{2 - 2\eta L} \right) = \delta\left(1 + \frac{\eta L}{2(1 - \eta L)} \right) = \Delta. 
\end{align*}
Hence, the desired result is achieved. 

\subsection{Proof of Corollary \ref{cor:thm_3_complexity}}

Based on Theorem \ref{thm:generalconvex_02}, if we would aim for an $\epsilon$-accuracy solution, we can choose $\Delta = \epsilon/4$ and $\alpha=1/2$ (with $\eta = 2/(3 L)$). To obtain the convergence to an $\epsilon$-accuracy solution, we need to have $\delta = \Ocal(\epsilon)$, or equivalently, $m = \Ocal(1/\epsilon)$. Then we have
\begin{align*}
\mathbb{E}[ \| \nabla P(\tilde{w}_s)\|^2 ]
&\overset{\eqref{asdfsfas}}{\leq} \frac \Delta2 + \frac12 \mathbb{E}[ \| \nabla P(\tilde{w}_{s-1})\|^2 ] \\
&\leq \frac \Delta 2 + \frac\Delta{2^2} +\frac1{2^2} \mathbb{E}[ \| \nabla P(\tilde{w}_{s-2})\|^2 ]
\\
&\leq  \Delta \left(\frac12+\frac1{2^2}+\dots+\frac1{2^s}\right) +\frac1{2^s} \| \nabla P(\tilde{w}_{0})\|^2 \\
&\leq \Delta +\frac1{2^s} \| \nabla P(\tilde{w}_{0})\|^2.
\end{align*}

 To guarantee that 
$\mathbb{E}[ \| \nabla P(\tilde{w}_s)\|^2 ]\leq \epsilon$, 
it is sufficient to make
$\frac1{2^s} \| \nabla P(\tilde{w}_{0})\|^2 \leq \tfrac34 \epsilon$, or
$s = \Ocal(\log (1/\epsilon))$. This implies the total complexity to achieve an $\epsilon$-accuracy solution is $(n+2m)s = \Ocal((n + (1/\epsilon))\log(1/\epsilon))$.

\subsection{Proof of Corollary~\ref{cor:complexity}}
Based on Lemma~\ref{lem:complexity} and Theorem \ref{thm:stronglyconvexconvergence}, let us pick $\theta^* = 2$, i.e, then we have $m^* = 4.5\kappa -1.$ So let us run SARAH with $\eta = 1/(2L)$ and $m=4.5\kappa$, then we can calculate $\sigma_m$ in \eqref{eq:sigma0} as
$$\sigma_m = \frac{1}{\mu \eta (m + 1)} +  \frac{\eta L}{2 - \eta L}  = \frac{1}{[\mu/(2L)] (4.5\kappa + 1)} +  \frac{1/2}{2 - 1/2}  <  \frac{4}{9} + \frac{1}{3} = \frac{7}{9}.$$
 According to Theorem \ref{thm:stronglyconvexconvergence}, if we run SARAH for $\mathcal{T}$ iterations where
 $$\mathcal{T} = \lceil \log (\|\nabla P(\tilde{w}_0)\|^2 /\epsilon)/\log(9/7)\rceil = \lceil \log_{7/9}(\epsilon/\|\nabla P(\tilde{w}_0)\|^2)\rceil \geq \log_{7/9}(\epsilon/\|\nabla P(\tilde{w}_0)\|^2),$$
then we have
 $$ \Exp[\|\nabla P(\tilde{w}_\mathcal{T})\|^2] \leq (\sigma_m)^\mathcal{T} \|\nabla P(\tilde{w}_0)\|^2 < (7/9)^\mathcal{T}\|\nabla P(\tilde{w}_0)\|^2 \leq (7/9)^{\log_{7/9}(\epsilon/\|\nabla P(\tilde{w}_0)\|^2)}\|\nabla P(\tilde{w}_0)\|^2= \epsilon,$$
 thus we can derive \eqref{eq:accuracy}.
If we consider the number of gradient evaluations as the main computational complexity, then the total complexity can be obtained as 
$$(n+2m) \mathcal{T} = \Ocal\left((n+\kappa)\log(1/\epsilon)\right).$$

%
%
%
%
%
%
%
%
%

\end{document}